\documentclass[twocolumn, journal]{IEEEtran}

\usepackage{amsmath,amssymb,amsfonts}
\usepackage{amsthm}
\usepackage{bm}
\usepackage{booktabs}
\usepackage{graphicx}
\usepackage{cite}
\usepackage{xcolor}
\usepackage{algorithm}
\usepackage{algorithmic}      
\usepackage{verbatim}         
\usepackage[hidelinks]{hyperref}

\newtheorem{theorem}{Theorem}
\newtheorem{lemma}{Lemma}
\newtheorem{assumption}{Assumption}

\title{Priority-Aware Learning-Unlearning Correction\\for Dynamic Decentralized LoRA Fine-Tuning}

\author{Nuocheng~Yang, Yechen He, \emph{Student Member, IEEE}, Sihua~Wang, Zihan~Chen, \emph{Member, IEEE}, \\ Tony Q. S. Quek, ~\IEEEmembership{Fellow,~IEEE}, and~Changchuan~Yin, \emph{Senior Member, IEEE}.
\thanks{N. Yang, Y. He, S. Wang, and C. Yin are with the Beijing Laboratory of Advanced Information Network, and the Beijing Key Laboratory of Network System Architecture and Convergence, Beijing University of Posts and Telecommunications, Beijing 100876, China (emails: \{yangnuocheng, heyechen, sihuawang, ccyin\}@bupt.edu.cn).}
\thanks{Z. Chen and T. Q.~S. Quek are with the Information Systems Technology and Design Pillar, Singapore University of Technology and Design, 487372, Singapore (emails: zihan\_chen@mymail.sutd.edu.sg, tonyquek@sutd.edu.sg).}
}

\begin{document}

\maketitle
\begin{abstract}
As large language models (LLMs) are increasingly deployed at the network edge to provide pervasive generative AI services, decentralized federated learning (DFL) provides a vital mechanism for privacy-preserving, domain-specific fine-tuning through peer-to-peer exchanges of parameter-efficient updates.
However, the dynamic nature of practical decentralized edge networks,  where devices may dynamically join or leave the collaborative training process, requires the system to continuously adapt to new data while selectively removing prior contributions. 
This correction process remains a significant bottleneck, as individual device updates become deeply entangled within the global fine-tuned parameters.
To address this challenge, we propose a priority-aware learning-unlearning correction framework based on orthogonal LoRA that can enhance the knowledge evaluation through topology adjustment.
Specifically, we first design an orthogonal LoRA mechanism that yields post-training contribution coordinates, enabling history-free projection addition and deletion in response to membership changes. 
We then analyze the correction bottleneck and develop a priority-aware policy that selects among topology refinement, local correction, proximal damping, and synchronization scheduling according to the dominant residual term.
A resource allocation algorithm is further developed to allocate limited communication across layer groups, prioritizing the primary bottlenecks within per-round wireless constraints.
Experiments demonstrate that the proposed framework achieves robust post-event correction for both device join and leave events and validate that different residual regimes necessitate distinct correction actions.
\end{abstract}


\section{Introduction}
Large language models (LLMs) are reshaping intelligent services by providing powerful language understanding, reasoning, and generation capabilities across a wide range of applications, such as general-purpose task solving, code intelligence, autonomous agents, healthcare, and networking~\cite{GPT3,CoT,CodeT5Plus,LLMAgentSurvey,LargeLanguageModelsforNetworking}.
As LLM-enabled services increasingly operate over user-specific and domain-specific data at the network edge, adapting LLMs to distributed edge environments has become an important direction for next-generation edge intelligence. 
However, such adaptation is challenging because private-domain data are naturally distributed across edge devices, while centralized fine-tuning incurs prohibitive computation and communication costs and raises serious privacy concerns. 
Low-rank adaptation (LoRA)~\cite{LoRA} alleviates this burden by freezing the pretrained base model and training only lightweight low-rank adapter matrices, thereby substantially reducing per-device computation and communication overhead. 
Building on this advantage, collaborative LoRA fine-tuning has emerged as a practical paradigm in which devices fine-tune LLMs through local training and federated or decentralized exchange of LoRA updates, without sharing raw private data~\cite{LargeLanguageModelsforNetworking,MobileLLaMA,FederatedLargeLanguageModel,FederatedFineTuningofLLMs,chen2023role}.
Based on this, collaborative fine-tuning, however, is rarely conducted over a fixed set of devices.
In practical edge environments, devices join with new task-specific knowledge and leave due to mobility, energy constraints, or privacy withdrawal requests under regulations such as GDPR~\cite{sok_fu,fu_survey_liu}.
Each membership change alters the optimization objective. Newly joined devices introduce knowledge to be absorbed, while departing devices require their learned contributions to be removed.
The system must therefore correct the model after each event rather than merely continue training under a fixed objective.

This post-event correction faces three fundamental challenges.
First, the optimization target itself is unstable because each membership event shifts the optimal model, yet the limited per-round communication budget prevents the system from tracking this shift quickly, so the correction process accumulates error from outdated targets at every step.
Second, device parameters are tightly coupled in the shared LoRA adapter. Because individual contributions are inseparable after decentralized training, forgetting a departing device either leaves residual influence or inadvertently damages other devices' knowledge, while a newly joined device's parameters interfere with existing ones.
Third, convergence is inherently slow under communication constraints arising from the parameter-exchange topology design, where each round transmits only a small set of LoRA parameters over bandwidth-limited device-to-device links, requiring many rounds to converge, yet the system has a finite correction budget and cannot afford unbounded iterations.

A number of works extend parameter-efficient fine-tuning to distributed settings by introducing decentralized federated LoRA methods~\cite{FedALoRA,FedHeLLo,FedQuad,ghiasvandDecentralizedLowRankFineTuning2025,leeFedSVDAdaptiveOrthogonalization2025,yang2025lora}.
They design adaptive aggregation, heterogeneous LoRA allocation, orthogonalization, and quantization strategies to improve training efficiency under wireless constraints.
However, when devices dynamically join and leave, these methods encounter two concrete difficulties.
First, the shared adapter parameters form an indivisible aggregate, and there is no mechanism to isolate and extract a single device's contribution.
A departing device forces the system to either discard the entire adapter (losing all other devices' knowledge) or retain stale influence (violating the removal request).
Second, without a per-device contribution index, newly joined devices cannot be efficiently incorporated by identifying compatible directions in the existing adapter space, so the system must either retrain from scratch or risk interference between old and new knowledge.
Thus, in dynamic membership scenarios, these methods lack both the isolation granularity and the incorporation mechanism that post-event correction demands.
Departure-side methods, namely federated and decentralized unlearning, aim to remove client or data influence without full retraining~\cite{federaser,yeHeterogeneousDecentralizedMachineUnlearning2023,liuRethinkingMachineUnlearning2024}.
They propose model-agnostic meta-unlearning~\cite{federaser}, influence-function-based removal~\cite{fedu_influence}, knowledge distillation~\cite{kd_fedul}, and certified deletion protocols~\cite{certified_du_2026}.
LLM unlearning further studied to address the retain-forget tradeoff in language models~\cite{zhongUnlearningKnowledgeOverwriting2025}.
Yet, these unlearning methods assume infrastructure that is absent in a fully decentralized LoRA system.
Federated unlearning requires a central server to coordinate the forgetting process~\cite{fedrecover}. Without such a coordinator, the remaining devices cannot agree on a consistent removal target.
Historical update trajectories stored on a server enable precise influence estimation, but in a D2D-only network, each device retains only its current adapter state.
Auxiliary teacher models used for distillation-based forgetting are unavailable when devices share only low-rank adapters.
As a result, directly applying these protocols would either incur prohibitive synchronization overhead or produce inconsistent forgetting across devices, breaking the correctness guarantees that their centralized operation relies on.
Join-side methods address the complementary problem of absorbing new devices or tasks into an existing model.
Federated continual learning~\cite{fedweit,fccl} mitigates catastrophic forgetting when new clients with novel data distributions arrive, cold-start federated training~\cite{cold_start_fl} initializes newcomers without disrupting the existing model, and flexible participation protocols~\cite{wei2023flexible,ozkara2023distributed} analyze convergence under arbitrary client churn patterns.
In the parameter-efficient space, LoRA composition methods such as LoRAHub~\cite{LoRAHub} and orthogonal subspace merging~\cite{orthogonal_subspace_merging} enable modular combination of independently trained adapters.
However, these join methods address only one direction of membership change.
Federated continual learning and cold-start protocols focus exclusively on integrating new knowledge. When a device later departs, they offer no mechanism to precisely remove its contribution without retraining the remaining model from scratch.
LoRA composition methods merge adapters additively, so subtraction requires retaining the exact pre-merge checkpoint and the adapter being removed, information that decentralized devices do not store after aggregation.
More fundamentally, once adapters are aggregated through decentralized aggregation, individual contributions become inseparable regardless of how they were initially incorporated, so any join-only method is structurally incapable of supporting subsequent unlearning.

Topology-aware decentralized learning explains how network aggregation and gossip communication affect convergence~\cite{dfl_survey,lian2017can,topology_spectral_beyond}.
These methods characterize the relationship between the aggregation matrix's spectral gap and the optimization convergence rate, and they optimize communication topology to accelerate training under bandwidth constraints~\cite{OntheBenefitsofMultipleGossipStepsinCommunicationConstrainedDecentralizedFederatedLearning,DiameterConstrainedTopology}.
The critical limitation is that these topology methods are designed to accelerate consensus toward a fixed training target.
When a membership event changes the optimization objective, accelerating consensus toward the old target becomes actively harmful, since devices would converge faster to an outdated model.
Moreover, topology optimization allocates edges to maximize training throughput, but in post-event correction, communication resources must serve a different purpose, propagating the correction signal rather than the training signal.
In the nonconvex LoRA correction regime, dense topologies risk propagating locally biased updates across the network before local optimization can correct them, amplifying rather than reducing the post-event error.
Topology methods, therefore, lack the ability to adapt their communication topology to the dominant error source, which in post-event correction may be local optimization bias or client drift rather than consensus disagreement.


To fill these gaps, this paper proposes a priority-aware learning-unlearning correction framework for dynamic decentralized LoRA fine-tuning.
The core idea is to equip each device with a contribution identifier that persists across membership changes, so that individual device influences can be approximately isolated after decentralized training.
Under this identifier, a leaving device can be removed without historical gradient records, and a newly joined device can extend the existing model without overwriting other devices' knowledge.
We then formulate the post-event correction as an optimization problem under per-round communication constraints and analyze how local residual, consensus residual, and heterogeneity residual each contribute to the finite-round correction error.
This analysis reveals which error source dominates under different conditions, informing a corrective policy that allocates communication edges and local computation steps accordingly.

The main contributions are summarized as follows.
\begin{itemize}
\item We formulate a unified learning-unlearning event objective for dynamic, decentralized LoRA fine-tuning that captures both device join and leave events under per-round communication constraints.
\item We introduce a frozen random orthogonal LoRA projection mechanism that turns each device-specific basis into a post-training contribution index. For join events, the orthogonal basis naturally extends the existing subspace, allowing new knowledge to be incorporated without interfering with existing devices' coordinates. For leave events, this enables no-history projection deletion without affecting other devices' knowledge. 
\item We analyze the correction gap to identify that the dominant factors are per-group curvature and post event initial gap. We then transform these two quantities into computable rank-order priority proxies based on Fisher information for curvature, post-projection gradient energy for initial gap that can enable a priority-based allocation of local steps, proximal damping, and mixing density across layer groups without solving the full bound.
\end{itemize}
\section{Related Work}
\subsection{Networked LLM and LoRA-Based Collaborative Fine-Tuning}
Networked LLM studies establish why LLM adaptation is becoming a communication-system problem. The authors in \cite{LargeLanguageModelsforNetworking,MobileLLaMA,FederatedLargeLanguageModel,FederatedFineTuningofLLMs} investigated LLMs for networking, mobile network analysis, and federated LLM fine-tuning. The authors in \cite{Communication-EfficientWirelessFederatedFine-TuningforLarge-ScaleAIModels,DataDivergenceawareClientSelection,chen2025zeroth} further considered wireless fine-tuning and client selection for LLMs. The common assumption in these studies is that the main difficulty is efficient training or serving over a network. In a dynamic decentralized system, however, the device population itself changes, so the model must be corrected after join and leave events rather than only trained under a fixed population.
Parameter-efficient LoRA methods reduce the cost of such collaborative adaptation. The authors in \cite{LoRA} introduced low-rank adaptation by freezing the base model and updating a small set of adapter parameters. The authors in \cite{FedALoRA,FedHeLLo,FedQuad,ADFlora} proposed adaptive aggregation, heterogeneous LoRA allocation, layer-wise deployment, and alternating low-rank aggregation for federated or decentralized fine-tuning. The authors in \cite{leeFedSVDAdaptiveOrthogonalization2025,FedQLoRA,adaptive_fedlora,lora_fair} further studied orthogonalization, quantization-aware LoRA, wireless heterogeneous LoRA, and fair aggregation. These works make collaborative LLM fine-tuning feasible, but the adapter coordinates are designed for training efficiency, not for identifying which part of the trained adapter belongs to a leaving device.
Adapter heterogeneity and interference have also been studied in multi-task and peer-to-peer settings. The authors in \cite{FlyLoRA,THANORA,LoRI,OMoE,FFTMoE} studied rank-wise mixture, task-heterogeneity-aware adaptation, cross-task interference, orthogonal LoRA mixtures, and sparse foundation-model fine-tuning. The authors in \cite{ghiasvandDecentralizedLowRankFineTuning2025,dec_lora,LearningtoCollaborate,TokenLevelLLMCollaboration} investigated decentralized low-rank fine-tuning and peer-to-peer LLM collaboration, while \cite{yang2025lora} considered sparse-and-orthogonal LoRA for wireless multi-task LLM fine-tuning. These studies show that adapter subspaces are not neutral averaging variables. This paper uses that insight differently: the LoRA basis is fixed as a contribution coordinate so that membership-driven unlearning can be followed by decentralized correction.
\subsection{Federated, Decentralized, and LLM Unlearning}
Federated unlearning aims to remove client or data influence without full retraining. The authors in \cite{federaser,kd_fedul,fedrecover,fedu_influence} proposed client-level removal, knowledge-distillation-based unlearning, recovery from historical updates, and influence-approximation-based federated unlearning. The surveys in \cite{sok_fu,fu_survey_liu} summarized federated unlearning methods and open challenges. These methods clarify the deletion objective, but many of them rely on a server, stored update trajectories, retraining-like recovery, or an auxiliary teacher. These requirements are difficult to satisfy when devices only hold their current decentralized LoRA adapters.
Decentralized unlearning removes the central server assumption and is therefore closer to the target scenario. The authors in \cite{gong2021bayesian} studied Bayesian variational federated learning and unlearning in decentralized networks. The authors in \cite{yeHeterogeneousDecentralizedMachineUnlearning2023,certified_du_2026,rr_du_2025,linDecentralizedUnlearningTrustworthy2024} studied heterogeneous decentralized unlearning, certified decentralized unlearning, fully decentralized certified unlearning, and decentralized AIGC-service unlearning. These works demonstrate that unlearning can be performed without a single coordinator, but they often add distillation, certification, ensemble, or trust structures. In contrast, this paper focuses on a lighter LoRA correction setting where the available information is the current adapter state and the frozen device bases.
LLM unlearning introduces an additional retain-forget conflict. The authors in \cite{liuRethinkingMachineUnlearning2024,zhongUnlearningKnowledgeOverwriting2025,hfu_llm,fade,o3_continual_unlearning} studied LLM unlearning, sparse-adapter knowledge overwriting, hierarchical federated LLM unlearning, selective forgetting, and continual LLM unlearning. Training-free and influence-based ideas in \cite{huynhFastFedULTrainingFree2024,influence_functions} further motivate approximate deletion without full retraining. These works support the need to compare the post-unlearning model with a retraining or retain-aware oracle. What remains open is how to connect this retain-forget objective with decentralized topology and finite-round LoRA correction after membership events.
\subsection{Decentralized Learning, Gossip, and Topology-Aware Optimization}
Decentralized learning provides the algorithmic basis for serverless model adaptation. The authors in \cite{fedavg,dfl_survey,lian2017can,Lalitha2018FullyDF} studied communication-efficient federated learning, decentralized federated learning fundamentals, decentralized stochastic gradient methods, and fully decentralized federated learning. The authors in \cite{ProceedingsIEEE,HVPoor,DecentralizedFederatedLearningBalancingCommunicationandComputingCosts} analyzed communication-computation tradeoffs and decentralized learning costs. These works show that local computation and communication are coupled, but the analyzed objective is typically a stationary training loss.
Gossip and topology-aware optimization explain the consensus part of decentralized convergence. The authors in \cite{topology_spectral_beyond,OntheBenefitsofMultipleGossipStepsinCommunicationConstrainedDecentralizedFederatedLearning,DiameterConstrainedTopology,Expandergraph} studied topology effects beyond spectral gap, multiple gossip steps, diameter-constrained topology orchestration, and expander-graph-based decentralized optimization. The authors in \cite{PeertoPeerVariationalFederatedLearningOverArbitraryGraphs,DecentralizedandModelFreeFederatedLearningConsensusBasedDistillationinFunctionSpace,DecentralizedFederatedLearningBalancingCommunicationandComputingCosts} investigated peer-to-peer variational learning, consensus-based distillation, and communication-computation tradeoffs over arbitrary networks. These studies justify using topology when disagreement dominates, but they also imply a limitation: faster aggregation is only one correction action and may not solve local drift or retain-forget conflict.
Recent decentralized systems further consider heterogeneity, reliability, and resource-aware topology design. The authors in \cite{topology_dfl_bandwidth,topology_dfl_energy,FedDual,EdgeBasedCommunicationOptimization,AcceleratingDFLinHEC} studied bandwidth-aware topology optimization, energy-efficient topology design, pair-wise gossip, edge-based communication optimization, and acceleration in heterogeneous edge computing. The authors in \cite{FutureInternet,muDFL,ReputationAwareHedonicCoalitionFormationforEfficientServerlessHierarchicalFederatedLearning,OpportunitiesofFederatedLearninginConnectedCooperativeandAutomatedIndustrialSystems,ByzantineAttack} considered fairness-aware peer-to-peer learning, secure microchained DFL, reputation-aware coalitions, connected industrial systems, and Byzantine-resilient decentralized SGD. 
\subsection{Wireless Resource Allocation for Networked Learning}
Wireless edge learning turns correction design into a constrained resource-allocation problem. The authors in \cite{wang2023towards,AJointLearning,Wireless_Communications_for_Collaborative_Federated_Learning,CGHS} studied cooperative federated learning over heterogeneous edge/fog networks, joint learning-communication design, wireless communications for collaborative federated learning, and distributed learning over wireless networks. The authors in \cite{liuCommunicationEnergyEfficient2023,yang2024gnn,yang2024bayesian} further studied communication-energy-efficient decentralized learning, graph-neural-network-based collaborative FL energy optimization, and secure distributed Bayesian FL. These works show that wireless constraints cannot be treated as an afterthought, but the learning task remains standard model training rather than membership-event learning-unlearning correction.
Graph-based wireless optimization provides scalable tools for link and resource decisions. The authors in \cite{GraphNeuralNetworksforWirelessCommunicationsFromTheorytoPractice,GraphEmbedding_BasedWirelessLinkSchedulingWithFewTrainingSamples,AGNN_BasedSupervisedLearningFrameworkforResourceAllocationinWirelessIoTNetworks,WirelessLinkSchedulingviaGraphRepresentationLearningAComparativeStudyofDifferentSupervisionLevels} studied graph neural networks for wireless communications, graph-embedding-based link scheduling, GNN-based resource allocation, and supervision levels for wireless link scheduling. The authors in \cite{CooperativeTrajectoryDesignofMultipleUAVBaseStationsWithHeterogeneousGraphNeuralNetworks,OptimalWirelessResourceAllocationWithRandomEdgeGraphNeuralNetworks,GraphNeuralNetworksforScalableRadioResourceManagementArchitectureDesignandTheoreticalAnalysis} studied UAV trajectory design, random-edge GNN resource allocation, and scalable radio resource management. These methods can optimize links once the network objective is known, but they do not specify whether a post-event LoRA system should prioritize edges, local steps, proximal damping, or synchronization.
Hierarchical and mobility-aware systems further illustrate why dynamic correction is necessary. The authors in \cite{NetworkSupportforHigh_PerformanceDistributedMachineLearning,Fedave,HierarchicalFederatedLearningACROSSHeterogeneousCellularNetworks,DecentralizedWirelessFederatedLearningWithDifferentialPrivacy,Decentralizedfederatedlearningthroughproxymodelsharing} studied network support for distributed machine learning, vehicular edge FL, hierarchical FL, decentralized wireless FL with differential privacy, and proxy-model sharing. These works capture realistic deployment pressures, while this paper focuses on the missing correction logic: when the membership changes, the wireless network should support not only faster training communication but also faster learning-unlearning recovery.
\begin{table*}[t]
\centering
\caption{Comparison between representative related works and the considered dynamic decentralized LoRA learning-unlearning correction problem.}
\label{tab:tnse_related_comparison}
\footnotesize
\resizebox{\textwidth}{!}{%
\begin{tabular}{p{2.4cm}p{4.2cm}cccccp{4.4cm}}
\toprule
Work category & Representative works & LLM/LoRA & Dynamic join/leave & Unlearning & No-history removal & Decentralized constraint & Main limitation for this paper \\
\midrule
Federated and decentralized LoRA fine-tuning &
\cite{LoRA,FedALoRA,FedHeLLo,FedQuad,ADFlora,ghiasvandDecentralizedLowRankFineTuning2025,leeFedSVDAdaptiveOrthogonalization2025,yang2025lora} &
Yes & Limited & No & No & Partial &
They reduce adapter training and communication costs, but do not provide a post-training device contribution index for leave events. \\
\midrule
LLM collaboration and adapter interference control &
\cite{FlyLoRA,THANORA,LoRI,OMoE,FFTMoE,LearningtoCollaborate,TokenLevelLLMCollaboration,yang2025lora} &
Yes & Limited & Limited & No & Partial &
They study collaboration, mixture, and interference in adapter spaces, but do not formulate dynamic learning-unlearning correction. \\
\midrule
Federated and decentralized unlearning &
\cite{federaser,kd_fedul,fedrecover,fedu_influence,sok_fu,fu_survey_liu,gong2021bayesian,yeHeterogeneousDecentralizedMachineUnlearning2023,certified_du_2026,rr_du_2025} &
Partial & Limited & Yes & Usually no & Partial &
They remove data or client influence, but often require history, server-side control, distillation, certification, or retraining-like signals. \\
\midrule
LLM and adapter unlearning &
\cite{liuRethinkingMachineUnlearning2024,zhongUnlearningKnowledgeOverwriting2025,hfu_llm,fade,o3_continual_unlearning,huynhFastFedULTrainingFree2024,linDecentralizedUnlearningTrustworthy2024} &
Yes & Limited & Yes & Partial & Limited &
They address retain-forget behavior, but do not jointly handle device-level contribution indexing and decentralized finite-round correction. \\
\midrule
Decentralized learning and gossip topology &
\cite{fedavg,dfl_survey,lian2017can,Lalitha2018FullyDF,topology_spectral_beyond,OntheBenefitsofMultipleGossipStepsinCommunicationConstrainedDecentralizedFederatedLearning,DiameterConstrainedTopology,Expandergraph} &
No/Partial & Partial & No & No & Yes &
They explain consensus and topology effects for fixed training objectives, but not post-event learning-unlearning residual decomposition. \\
\midrule
Wireless networked learning and resource allocation &
\cite{wang2023towards,AJointLearning,Wireless_Communications_for_Collaborative_Federated_Learning,CGHS,liuCommunicationEnergyEfficient2023,yang2024gnn,GraphNeuralNetworksforWirelessCommunicationsFromTheorytoPractice,GraphEmbedding_BasedWirelessLinkSchedulingWithFewTrainingSamples} &
Partial & Partial & No & No & Yes &
They optimize communication, energy, and wireless links, but do not decide whether topology or local correction should be prioritized after membership events. \\
\midrule
Proposed framework &
This paper &
Yes & Yes & Yes & Yes & Yes &
It uses frozen orthogonal LoRA bases for no-history contribution indexing and a Fisher-driven priority allocation framework that distributes correction resources across layer groups. \\
\bottomrule
\end{tabular}
}
\end{table*}
\section{System Model and Problem Formulation}
\begin{figure}[t]
\centering
\includegraphics[width=\columnwidth]{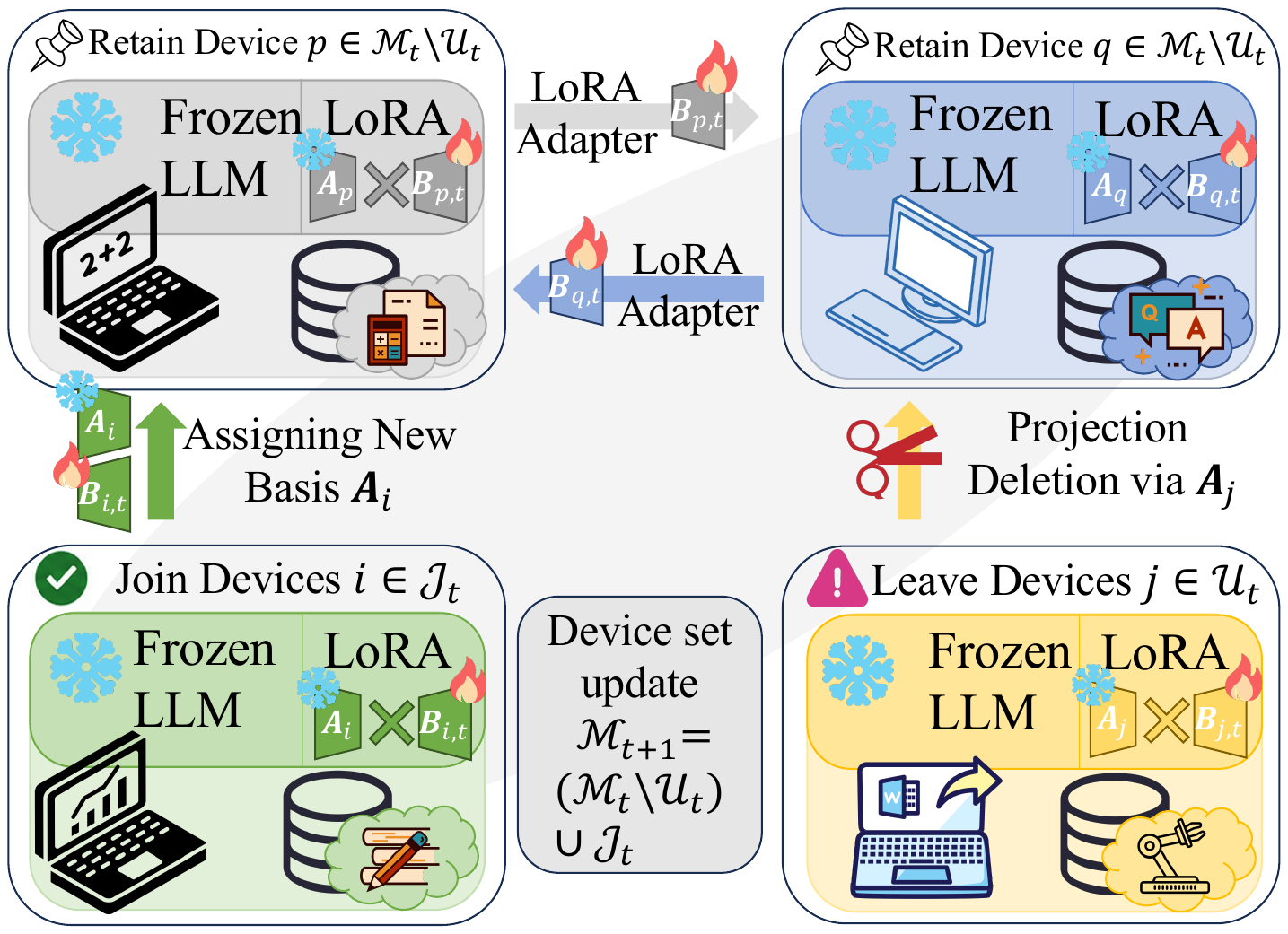}
\caption{Illustration of the considered dynamic DFL framework.}
\label{fig:sysmodel}
\end{figure}
\subsection{Dynamic Decentralized LoRA System}
Consider a distributed wireless network in which a set of edge devices collaboratively fine-tune an $L$-layer LLM without relying on a central parameter server. 
Devices exchange only the fine-tuned LoRA adapter parameters via device-to-device links, rather than raw data.
Unlike the traditional system model, which assumes a fixed set of devices that attend to collaborative training, the members who attend collaborative fine-tuning change across training epochs. 
Specifically, at epoch $t$, the active devices' set is $\mathcal{M}_t=\{1,2,\ldots,M_t\}$.
Each device $i\in\mathcal{M}_t$ holds a private local dataset $\mathcal{D}_{i,t}$ which may correspond to different downstream tasks or data distributions. 
To reduce the computational overhead of full-parameter training, a LoRA-based method is introduced, where each device freezes the $L$ layer pretrained model $\bm{\theta}_0=[\bm{\theta}_{0,1},\cdots,\bm{\theta}_{0,L}]$ and trains an additional low-rank adapter per layer.
For layer $\ell$ with pre-trained parameter $\bm{\theta}_{0,\ell} \in \mathbb{R}^{d_\ell \times k_\ell}$, the low-rank adapted parameter is given by
\begin{equation}
    \Delta \bm{W}_{i,t}^{(\ell)} = \bm{B}_{i,t}^{(\ell)}\bigl(\bm{A}_i^{(\ell)}\bigr)^\top,
    \label{eq:localLoRA}
\end{equation}
Specifically, in the proposed method, $\bm{A}_i^{(\ell)} \in \mathbb{R}^{d_\ell \times r_\ell}$ is a device-specific, frozen projection matrix sampled independently from a zero-mean, unit-variance Gaussian distribution.
$\bm{B}_{i,t}^{(\ell)} \in \mathbb{R}^{k_\ell \times r_\ell}$ is the trainable expansion matrix and $r_\ell \ll \min(d_\ell, k_\ell)$.
Since $\bm{A}_i^{(\ell)}$ is frozen, only $\bm{B}_{i,t}^{(\ell)}$ is trained and exchanged between devices during the collaborative fine-tuning process.
For an local stored input-output data pair $(x,y) \in \mathcal{D}_{i,t}$, the autoregressive loss under $\bm{A}_i$ and $\bm{B}_{i,t}$ is given by
\begin{equation}
    \ell(y|x;\bm{\theta}_0,\bm{A}_i,\bm{B}_{i,t})
    = -\sum_{s=1}^{|y|} \log p_{\bm{\theta}_0,\bm{A}_i,\bm{B}_{i,t}}(y_s|x,y_{<s}),
\end{equation}
where $\bm{\theta}_0 = [\bm{\theta}_{0,1},\dots,\bm{\theta}_{0,L}]$, $\bm{A}_i = [\bm{A}_i^{(0)},\dots,\bm{A}_i^{(L)}]$, and $\bm{B}_{i,t} = [\bm{B}_{i,t}^{(0)},\dots,\bm{B}_{i,t}^{(L)}]$.
The fine-tuned $\ell$-th layer's parameter is $\bm{W}_{i,t}^{(\ell)} = \bm{\theta}_{0,\ell} + \Delta \bm{W}_{i,t}^{(\ell)}$.
The local training objective of device $i$ is
\begin{equation}
\mathcal{L}_{i,t}^{\mathrm{tr}}(\bm{B}_{i,t};\mathcal{D}_{i,t})
=
\mathbb{E}_{(x,y)\in\mathcal{D}_{i,t}}
\left[
\ell(y|x;\bm{\theta}_0,\bm{A}_i,\bm{B}_{i,t})
\right].
\end{equation}
During decentralized LoRA fine-tuning, each device alternates between $n_{\ell}$ steps of local updates and a model aggregation step from neighbors. 
Each local step updates the expansion matrix using gradient descent.

\subsection{Model Transmission and Aggregation}
After local fine-tuning, devices will exchange their LoRA parameters with a subset of neighbors.
Let $\mathcal{E}_t^{(\ell)}$ be the undirected model exchange edge set selected for layer $\ell$ at round $t$, and let $d_{i,t}^{(\ell)}$ be the degree of device $i$ in this graph (i.e., the number of neighbors to exchange model parameter). 
Device $i$'s model aggregated from neighbors through model exchange edge set $\mathcal{E}_t^{(\ell)}$ can be given by
\begin{equation}
\bm{B}_{i,t}^{(\ell)} = \sum_{j\in\mathcal{M}_{t}} \left[\bm{W}_t^{(\ell)}(\gamma_\ell)\right]_{i,j}\,\bm{B}_{j,t}^{(\ell)}.
\label{eq:aggregation_step_sys}
\end{equation}
The damped aggregation matrix is given by
\begin{equation}
\bm{W}_t^{(\ell)}(\gamma_\ell) = (1-\gamma_\ell)\bm{I} + \gamma_\ell\bm{W}_t^{(\ell)},
\label{eq:aggregation}
\end{equation}
where $\gamma_\ell\in[0,1]$ controls the blending between purely local correction ($\gamma_\ell=0$) and full neighbor aggregation ($\gamma_\ell=1$), and the base Metropolis matrix $\bm{W}_t^{(\ell)} = [w_{i,j,t}^{(\ell)}]_{i,j\in\mathcal{M}_{t+1}}$ can be given by
\begin{equation}
w_{i,j,t}^{(\ell)}
=
\begin{cases}
\dfrac{1}{\max\{d_{i,t}^{(\ell)},d_{j,t}^{(\ell)}\}}, & \text{ if } (i,j)\in\mathcal{E}_t^{(\ell)},\ i\neq j,\\[8pt]
1-\sum\limits_{m\in\mathcal{M}_t,m\neq i}w_{i,m,t}^{(\ell)}, & \text{ if } i=j,\\[8pt]
0, & \text{otherwise}.
\end{cases}
\end{equation}
Here, we can see that the design variable is the feasible edge set $\mathcal{E}_t^{(\ell)}$ and once it is selected, all model aggregation weights are determined.

\textcolor{black}{We further define the correction schedule of layer $\ell$ as $\Pi_\ell = (n_\ell,\eta_\ell,\lambda_\ell^{\mathrm{prox}},\gamma_\ell)$.
$\Pi_\ell$ consists of four quantities, which include the number of local steps $n_\ell$, the step size $\eta_\ell$, the proximal coefficient $\lambda_\ell^{\mathrm{prox}}$, and the aggregation strength $\gamma_\ell$, collectively specifying the correction schedule for layer $\ell$.
Let $s_\ell$ be the data size of LoRA scalars of layer $\ell$, and let $q_b$ be the number of quantization bits per scalar.
The data set of device $i$ transmitting its model to device $j$ is
$b_{i,j,t}^{(\ell)} = q_b s_\ell$. }

We adopt the OFDMA scheme for wireless model transmission.
The achievable transmission rate from device $i$ to device $j$ at round $t$ is given by
\begin{equation}
R_{i,j,t}
=
B_{i,j,t}^{\mathrm{ch}}
\log_2
\left(
1+
\frac{
p_{i,t}h_{i,j,t}
}{
\sigma^2+
\sum_{m\in\mathcal{I}_{j,t}}
p_{m,k}h_{mj,t}
}
\right),
\end{equation}
where $B_{i,j,t}^{\mathrm{ch}}$ is the available bandwidth, $p_{i,t}$ is the transmit power, $h_{i,j,t}$ being the channel gain. 
$\sigma^2$ is the noise power, and $\mathcal{I}_{j,t}$ being co-channel interference set.
The per-edge communication cost combines payload and transmission latency into a single budget measure as
\begin{equation}
c_{i,j,t}^{(\ell)}
=
b_{i,j,t}^{(\ell)}
+
\beta_{\tau}\frac{b_{i,j,t}^{(\ell)}}{R_{i,j,t}},
\end{equation}
where $\beta_{\tau}\geq0$ converts latency into the same budget unit as payload, $b_{i,j,t}^{(\ell)}$ is the transmission data size.

\subsection{Unified Membership Event Objective}
After the local training and aggregation stage, the system may experience a membership event. 
Specifically, at epoch $t$, a set $\mathcal{J}_t$ of $J_t$ devices joins the system with fully initialized model parameters to train on their local datasets. 
There may also be a set $\mathcal{U}_t$ of $U_t$ devices that leave the system and request the removal of their contributions from the model parameters that are already trained.
In this scenario, the active device set at epoch $t+1$ can be given by
\begin{equation}
\mathcal{M}_{t+1}=(\mathcal{M}_t\setminus\mathcal{U}_t)\cup\mathcal{J}_t.
\end{equation}
This membership event is denoted by $e_t=(\mathcal{J}_t,\mathcal{U}_t)$.
Let $\mathcal{D}_{\mathrm{ret}}(e_t)$, $\mathcal{D}_{\mathrm{join}}(e_t)$, and $\mathcal{D}_{\mathrm{for}}(e_t)$ denote the retain, join, and forget sets induced by event $e_t$. 
For the forget set, $y_f$ denotes the desired post-unlearning response, such as a refusal, neutral, or sanitized target response. 
Thus, the final event objective is
\begin{align}
\mathcal{L}_{\mathrm{evt}}^{(e_t)}(\bm{B})
=&
\mathbb{E}_{(x,y)\in\mathcal{D}_{\mathrm{ret}}(e_t)}
\left[
\ell(y|x;\bm{\theta}_0,\bm{A},\bm{B})
\right]
\nonumber\\
&+
\lambda_j
\mathbb{E}_{(x,y)\in\mathcal{D}_{\mathrm{join}}(e_t)}
\left[
\ell(y|x;\bm{\theta}_0,\bm{A},\bm{B})
\right]
\nonumber\\
&+
\lambda_f
\mathbb{E}_{(x,y_f)\in\mathcal{D}_{\mathrm{for}}(e_t)}
\left[
\ell(y_f|x;\bm{\theta}_0,\bm{A},\bm{B})
\right],
\end{align}
where $\lambda_j\geq0$ and $\lambda_f\geq0$ control the learning and forgetting terms. If the event contains only a leave request, $\lambda_j=0$, if it contains only a join request, $\lambda_f=0$.

Let
$\bm{B}_{e_t}^{\star}
=
\arg\min_{\bm{B}}
\mathcal{L}_{\mathrm{evt}}^{(e_t)}(\bm{B})$
be the optimal adapter.
In particular, the oracle cannot be computed in the decentralized setting, so the goal is to approximate it via finite-round decentralized gradient descent (DGD) correction.
The event correction gap after $K$ DGD rounds is
\begin{equation}
\mathcal{E}_{\mathrm{evt}}(e_t,K)
=
\mathcal{L}_{\mathrm{evt}}^{(e_t)}(\bar{\bm{B}}_K)
-
\mathcal{L}_{\mathrm{evt}}^{(e_t)}(\bm{B}_{e_t}^{\star}),
\end{equation}
where $\bar{\bm{B}}_K = (\bar{\bm{B}}_K^{(1)},\dots,\bar{\bm{B}}_K^{(L)})$ and $\bar{\bm{B}}_K^{(\ell)} = \frac{1}{|\mathcal{M}_{t+1}|}\sum_{i\in\mathcal{M}_{t+1}} \bm{B}_{i,K}^{(\ell)}$ is the device-averaged expansion factor at layer $\ell$.

\textbf{Remark.}
Note that the forget term cannot be minimized directly by DGD, since the departing devices are unavailable.

\subsection{Problem Formulation}
We formulate our optimization problem whose goal is to minimize the event correction gap within $K$ correction rounds while jointly considering communication constraints by adjusting $\Pi_\ell$ and feasible communication graphs $\mathcal{E}_t^{(\ell)}$. 
The optimization problem is formulated as
\begin{equation}\label{eq:max1}
\min_{\{\Pi_\ell,\mathcal{E}_t^{(\ell)}\}}
\;
\mathcal{E}_{\mathrm{evt}}(e_t,K)
\end{equation}

\begin{align}\label{c1}
\mathrm{s.t.}\quad
&
\sum_{\ell=1}^{L}
\sum_{(i,j)\in\mathcal{E}_t^{(\ell)}}
c_{i,j,t}^{(\ell)}
\leq
C_t,
\quad
\forall t, \tag{\theequation a} \\
&
\frac{b_{i,j,t}^{(\ell)}}{R_{i,j,t}}
\leq
\tau_t^{\max},
\quad
\forall (i,j),g,t,
\tag{\theequation b}
\end{align}
where 
$C_t$ is the available per-round communication budget, and $\tau_t^{\max}$ is the maximum tolerable link latency in round $t$. 
The budget is not accumulated across rounds; every correction round must satisfy its own wireless constraint.
The optimization problem is hard to solve due to the following reasons.
First, the optimization problem is a mixed-integer program whose objective depends on the unknown event oracle $\bm{B}_{e_t}^{\star}$ and involves discrete graph variables coupled with continuous policy parameters.
Second, event membership causes unstable object changes, which introduce unstable training and parameter exchange.
Further, since the departing devices are unavailable, a data-free unlearning algorithm needs to be designed.

\section{Problem Analysis and Proposed Method}
To address problem (\ref{eq:max1}), in this section, we need to analyze the relationship between the event correction gap and derive a residual-level rule for selecting corrective actions.
Further, since the fine-tuning forget term cannot be minimized directly via gradient descent, we first propose a model initialization method that enables unlearning without using unlearning data.
\subsection{Orthogonal Projection Matrix Mechanism and Post-Event Initialization}
We begin by introducing a novel LoRA mechanism that enables a fast post-event initialization without storing the history gradient, which is commonly used in traditional works.
In particular, based on (\ref{eq:localLoRA}), each device $i$ is initialized with a frozen random orthogonal projection basis $\bm{A}_i^{(\ell)}\in\mathbb{R}^{d_\ell\times r_\ell}$ for each layer $\ell$ and remains fixed during training and aggregation, which can benefit the unlearning process.
To establish this, we first make the following assumption that formalizes the projection basis structure.
\begin{assumption}[Frozen random orthogonal bases]
For each device $i$ and layer $\ell$, $\bm{A}_i^{(\ell)}\in\mathbb{R}^{d_\ell\times r_\ell}$ has independently sampled elements and remains frozen during training and correction.
\end{assumption}
Then, we can analyze the impact between projection metric as the following Lemma.
\begin{lemma}[Orthogonal projection leakage]
Under Assumption 1, if $\bm{A}_i^{(\ell)}$ and $\bm{A}_j^{(\ell)}$ are independently generated by Gaussian sampling, then
\begin{equation}
\mathbb{E}
\left[
\left\|
\left(\bm{A}_i^{(\ell)}\right)^{\top}
\bm{A}_j^{(\ell)}
\right\|_F^2
\right]
=
\frac{r_\ell^2}{d_\ell}\approx \bm{0}, \text{ if } r_\ell \ll d_\ell.
\end{equation}
\end{lemma}
\begin{proof}
    See Appendix \ref{Lemm1}.
\end{proof}
Lemma~1 shows that randomly sampled $\bm{A}_i^{(\ell)}$ and $\bm{A}_j^{(\ell)}$ satisfy orthogonal bases that provide low-overlap device-specific coordinates rather than shared semantic directions.
This property is what enables contribution separation.
For each device $u$, we can define its projection operator and its orthogonal complement as
\begin{equation}
\bm{P}_u^{(\ell)}
=
\bm{A}_u^{(\ell)}
\left(\bm{A}_u^{(\ell)}\right)^{\top},
\quad
\bm{P}_{u,\perp}^{(\ell)}
=
\bm{I}-\bm{P}_u^{(\ell)}.
\end{equation}
When device $u$ leaves, each remaining device $i$ projects its LoRA adapter to remove the component aligned with $\bm{A}_u^{(\ell)}$ as
\begin{equation}\label{eq:eliment2}
\Delta \bm{W}_i^{(\ell)} = \Delta \bm{W}_i^{(\ell)} \, \bm{P}_{u,\perp}^{(\ell)}.
\end{equation}
Since the basis $\bm{A}_i^{(\ell)}$ is frozen, this is equivalent to re-initiation the expansion matrix as $\bm{B}_{i,0}^{(\ell)} = \bm{B}_{i,t}^{(\ell)} - \bm{B}_{i,t}^{(\ell)}(\bm{A}_i^{(\ell)\top}\bm{A}_u^{(\ell)})\bm{A}_u^{(\ell)\top}\bm{A}_i^{(\ell)}$, where the correction term is small because $\bm{A}_i^{(\ell)\top}\bm{A}_u^{(\ell)} \approx \bm{0}$.
Through (\ref{eq:eliment2}) device $i$ can eliminate device $u$'s contribution through that projection wheres device $u$'s knowledge learned during training is encoded in the subspace spanned by its frozen basis $\bm{A}_u^{(\ell)}$.
By projecting out this subspace from the shared adapter, the information contributed by $u$ is eliminated while other devices' knowledge, residing in approximately orthogonal subspaces, is preserved.
Thus, the proposed unlearning method can enable a stable subspace deletion without retaining.
Further, this operation is local, incurs no communication cost, and requires no historical gradients.

Similarly, when a new device $j$ joins, it generates its own orthogonal projection matrix $\bm{A}_j^{(\ell)}$ and initializes $\bm{B}_j^{(\ell)}$ from its local data.
Because $\bm{A}_j^{(\ell)}$ is approximately orthogonal to the existing bases (Lemma~1), the new adapter $\bm{B}_j^{(\ell)}(\bm{A}_j^{(\ell)})^\top$ extends the active subspace without interfering with previously learned knowledge.
The surviving devices' expansion matrixs $\bm{B}_{i,t}^{(\ell)}$ remain unchanged during this step.
Thus, both events are handled by a single mechanism as follows
\begin{itemize}
    \item \textbf{Device leave:} We can remove a device's subspace component via orthogonal projection.
    \item \textbf{Device join:} We can add a new orthogonal subspace component via direct initialization.
\end{itemize}

These benefits are both enabled by the frozen orthogonal basis structure.
To further quantify its impact, we have the following lemma, which bounds the residual transfer between different devices' subspaces and controls both how much of a leaving device's contribution leaks into other subspaces and how much a joining device's basis interferes with existing ones.
\begin{lemma}[Projection transfer residual]
Under Assumption 1, for any $i\neq j$ and layer $\ell$, the expansion matrix $\bm{B}_{i,t}^{(\ell)}$ satisfies
\begin{equation}
\left\|
\bm{B}_{i,t}^{(\ell)}
\left(\bm{A}_i^{(\ell)}\right)^{\top}
\bm{P}_j^{(\ell)}
\right\|_F^2
\leq
\|\bm{B}_{i,t}^{(\ell)}\|_F^2
\left(\left\|
\left(\bm{A}_i^{(\ell)}\right)^{\top}
\bm{A}_j^{(\ell)}
\right\|_2\right)^2.
\end{equation}
\end{lemma}
\begin{proof}
    See Appendix \ref{Lemm2}.
\end{proof}
Since each $\bm{A}_i^{(\ell)}$ is sampled independently, the overlap between different devices' projection matrices is provably limited.
When $j$ is a leaving device, this bound controls how much of device $i$'s adapter leaks into the leaving subspace after projection deletion.
When $i$ is a joining device, the bound controls the interference between the new basis and existing ones.
Together with Lemma~1 and Lemma~2, we can see that the orthogonal basis mechanism produces a well-defined post-event initialization based on frozen $\bm{A}_i^{(\ell)}$.
Prior work on projection-based unlearning~\cite{hoang2024learn,pan2025federated}, orthogonal subspace adaptation~\cite{wang2023orthogonal}, and project-then-recover pipelines~\cite{deng2024enable,li2026fedcare} supports the view that such an initialization falls within a recoverable neighborhood of the event objective.
\textcolor{black}{
Then, we can see that the orthogonal basis mechanism produces a well-defined post-event initialization without stored history.
Projection deletion bounds the leaving device's residual transfer, while orthogonal basis extension limits the joining device's interference with existing coordinates.
Both membership event types reduce to the same subsequent step: finite-round DGD correction starting from a recoverable neighborhood.
However, the per-round communication budget and finite correction horizon $K$ require allocating correction resources across layer groups according to each group's sensitivity.}
To further analyze the convergence of DGD correction from this initialization, we adopt standard regularity assumptions on the objective as in ~\cite{karimi2016linear,kuruzov2022gradient,yuan2016convergence}.
\begin{assumption}[Local PL-regularity after projection initialization]
For each layer $\ell$, $\mathcal{L}_{\mathrm{evt}}^{(e_t)}$ as a function of $\bm{B}^{(\ell)}$ (with other layers fixed) satisfies a local $\mu_\ell$-PL condition in a neighborhood that contains both the post-event initialization produced by the projection steps above and the centralized optimum model.
Each local correction objective $\mathcal{L}_{i,t}^{\mathrm{tr}}(\bm{B}_{i,t}^{(\ell)};\mathcal{D}_{i,t})$ is $L_\ell$-smooth, and $\mathcal{L}_{\mathrm{evt}}^{(e_t)}$ as a function of $\bm{B}^{(\ell)}$ (with other layers fixed) is also $L_\ell$-smooth~\cite{karimi2016linear,yuan2016convergence}.
The step size satisfies $0<\eta_c\leq 1/L_\ell$.
The local objective mismatch near the optimum is finite and can be given by
\begin{equation}
\tilde{\zeta}_\ell^2(e_t)
=
\frac{1}{|\mathcal{M}_{t+1}|}
\sum_{i\in\mathcal{M}_{t+1}}
\left\|
\nabla\mathcal{L}_{i,t}^{\mathrm{tr}}(\bm{B}_{i,t}^{(\ell)};\mathcal{D}_{i,t})
(\bm{B}^{(g),\star}(e_t);e_t)
\right\|_2^2.
\end{equation}
\end{assumption}
Under this assumption, the gap between the initialization after membership event and the event oracle in objective value is bounded by the parameter residual as
\begin{equation}
\mathcal{L}_{\mathrm{evt}}^{(e_t)}(\widetilde{\bm{B}}_0) - \mathcal{L}_{\mathrm{evt}}^{(e_t)}(\bm{B}_{e_t}^{\star})
\leq
\frac{L_\ell}{2}
\left\|
\widetilde{\bm{B}}_0^{(\ell)}(e_t) - \bm{B}_{e_t}^{\star,(g)}
\right\|_2^2 .
\end{equation}
After leave-side projection deletion, the pre-event consensus expansion factor $\bm{B}_t^{\mathrm{old},(\ell)}$ is updated to
\begin{equation}
\widetilde{\bm{B}}_{\mathrm{del}}^{(\ell)}(e_t)
=
\bm{B}_t^{\mathrm{old},(\ell)} - \bm{B}_t^{\mathrm{old},(\ell)}\bigl(\bm{A}^{(\ell)\top}\bm{A}_u^{(\ell)}\bigr)\bigl(\bm{A}_u^{(\ell)\top}\bm{A}^{(\ell)}\bigr),
\end{equation}
where the correction term is bounded by the basis leakage $\chi_{iu}^{(\ell)}$ from Lemma~2.

When a new device $j\in\mathcal{J}_t$ joins, it trains an initial expansion factor $\bm{B}_{j,0}^{(\ell)}$ from its local data.
The post-event expansion factor $\widetilde{\bm{B}}_{0}^{(\ell)}(e_t)$ is then the horizontal concatenation of the projected consensus and all joining devices' factors:
\begin{equation}
\widetilde{\bm{B}}_{0}^{(\ell)}(e_t) = \bigl[\,\widetilde{\bm{B}}_{\mathrm{del}}^{(\ell)}(e_t),\;\; \{\bm{B}_{j,0}^{(\ell)}\}_{j\in\mathcal{J}_t}\,\bigr],
\end{equation}
with the aggregate basis $\bm{A}_{\text{post}}^{(\ell)} = [\bm{A}_i^{(\ell)}]_{i\in\mathcal{M}_{t+1}}$.
Together, the two membership event types reduce to a well-defined finite-round correction problem.

The parameter residual from the post-event initialization to the event oracle is $\bm{r}_0^{(\ell)}(e_t) = \widetilde{\bm{B}}_0^{(\ell)}(e_t) - \bm{B}_{e_t}^{\star,(g)}$.
Under the $L_\ell$-smoothness condition in Assumption~2, the initialization gap admits a direct bound.
\begin{theorem}[Dynamic membership initialization gap]
Under Assumptions 1 and 2, after event $e_t$, the per-layer event objective gap before DGD correction satisfies
\begin{equation}
\mathcal{E}_{\mathrm{evt}}^{(\ell)}(e_t,0)
\leq
\frac{L_\ell}{2}\bigl\|\bm{r}_0^{(\ell)}(e_t)\bigr\|_2^2.
\end{equation}
\end{theorem}
\begin{proof}
    Since $\bm{B}_{e_t}^{\star}$ is a local optimum, $\nabla\mathcal{L}_{\mathrm{evt}}^{(e_t)}(\bm{B}_{e_t}^{\star})=\bm{0}$. By $L_\ell$-smoothness (Assumption~2),
    $\mathcal{L}(\widetilde{\bm{B}}_0) - \mathcal{L}(\bm{B}^\star) \leq \frac{L_\ell}{2}\|\widetilde{\bm{B}}_0 - \bm{B}^\star\|_2^2 = \frac{L_\ell}{2}\|\bm{r}_0\|_2^2$.
\end{proof}

Theorem~1 shows that the orthogonal basis mechanism converts both leave and join events into a bounded initialization gap proportional to $\|\bm{r}_0\|^2$.
For leave events, projection deletion provides a controlled starting point that approximately removes the departing device's contribution without requiring historical gradients.
For join events, the orthogonal structure ensures that a new device's basis does not significantly interfere with existing coordinates, so the join-side initialization gap remains bounded by the pairwise basis leakage quantified in Lemma~1.
Together, the two membership event types reduce to a well-defined finite-round correction problem.
\begin{figure}[t]
\centering
\includegraphics[width=7.5cm]{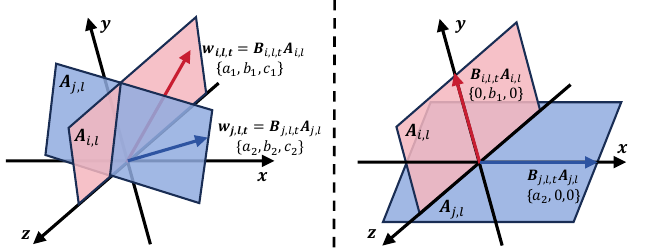}
\caption{Comparison of update conflicts in conventional LoRA and the proposed orthogonal LoRA method.}
\centering
\end{figure}

\subsection{Correction Gap Analysis under DGD}
Having obtained a bounded initialization, we now analyze how the correction gap evolves over rounds of DGD.
Starting from the initialization process, each device runs local correction steps and aggregation steps for $K$ rounds.
To model the parameter convergence, we model the following standard result, which bounds how much one aggregation step reduces disagreement across devices~\cite{lian2017can,yuan2016convergence,yang2024gnn}.
\begin{lemma}[Consensus contraction]\label{lem:consensus_contraction}
Let $\bm{W}\in\mathbb{R}^{M\times M}$ be doubly stochastic as defined in (\ref{eq:aggregation}) and $\bm{J}=\bm{1}\bm{1}^{\top}/M$.
For any $\bm{x}\in\mathbb{R}^{M\times d}$,
\begin{equation}
\|\bm{W}\bm{x} - \bm{J}\bm{x}\|_F
\leq
\|\bm{W} - \bm{J}\|_2 \cdot \|\bm{x} - \bm{J}\bm{x}\|_F.
\end{equation}
Moreover, for the Metropolis rule on a connected graph, $\|\bm{W} - \bm{J}\|_2 < 1$.
\end{lemma}
\begin{proof}
    See Appendix \ref{Lemm3}.
\end{proof}
Lemma~3 shows that $\|\bm{W} - \bm{J}\|_2$ is the contraction factor for device disagreement.
To capture the worst-case contraction across all $K$ correction rounds, we define the effective aggregation factor as
\begin{equation}
\bar{\rho}_\ell(e_t)
=
\max_{0\leq t<K}
\left\|
\bm{W}_t^{(\ell)}
-
\bm{J}
\right\|_2 .
\end{equation}
Specially, a smaller $\bar{\rho}_\ell(e_t)$ means faster consensus contraction for layer $\ell$, when $\bar{\rho}_\ell = 0$, all devices instantaneously convergence.
Under the $\mu_\ell$-PL condition, a gradient step with step size $\eta_c$ contracts the optimality gap by a factor of $(1-\eta_c\mu_\ell)$~\cite{karimi2016linear}.
After $n_\ell$ local steps between successive aggregation rounds, the cumulative contraction is defined as
\begin{equation}
q_{\mathrm{opt},\ell}
=
\left(1-\eta_c\mu_\ell\right)^{n_\ell}.
\end{equation}
The consensus term also depends on the amount of disagreement that already exists and the disagreement injected by local updates.
Define the loss-aware initial disagreement as
\begin{equation}
\mathcal{D}_0^{(\ell)}(e_t)
=
\frac{1}{|\mathcal{M}_{t+1}|}
\sum_{i\in\mathcal{M}_{t+1}}
\left\|
\bm{B}_{i,0}^{(\ell)}
-
\bar{\bm{B}}_{0}^{(\ell)}
\right\|_2^2,
\end{equation}
where $\bar{\bm{B}}_{0}^{(\ell)}$ is the device-averaged expansion matrix after initialization in layer $\ell$.
Before the $s$-th aggregation step, local correction can create new disagreement.
We measure this injection by
\begin{equation}
\mathcal{I}_s^{(\ell)}(e_t)
=
\!\!\frac{1}{|\mathcal{M}_{t+1}|}\!\!
\sum_{i\in\mathcal{M}_{t+1}}
\left\|
\nabla\mathcal{L}_{i,t}^{\mathrm{tr}}(\bm{B}_{i,t}^{(\ell)};\mathcal{D}_{i,t})
(\bm{B}_{i,t}^{(\ell)};e_t)
-
\bar{\bm{g}}_s^{(\ell)}
\right\|_2^2,
\end{equation}
where $\bar{\bm{g}}_s^{(\ell)} = \frac{1}{|\mathcal{M}_{t+1}|} \sum_{i\in\mathcal{M}_{t+1}} \nabla\mathcal{L}_{i,t}^{\mathrm{tr}}(\bm{B}_{i,t}^{(\ell)};\mathcal{D}_{i,t})$ is the average gradient.
These definitions separate three mechanisms that influence the correction gap.
$q_{\mathrm{opt},\ell}$ measures local optimization contraction where larger $q_{\mathrm{opt},\ell}^{2K}$ means the local residual decays faster with rounds.
$\bar{\rho}_\ell$ measures consensus contraction where smaller $\bar{\rho}_\ell$ means devices agree faster.
$\tilde{\zeta}_\ell^2$ and $\mathcal{I}_s^{(\ell)}$ measure heterogeneity and drift where larger values mean devices' local objectives point in different directions.
Then, we prove the learning-unlearning convergence bound of the gap in the following theorem, which combines these mechanisms.
\begin{theorem}[Correction gap decomposition]\label{Theorem2}
Under Assumptions 1--2 and Theorem~1, after $K$ DGD correction rounds, there exist constants $c_{\mathrm{con},\ell}>0$ and $c_{\mathrm{het},\ell}>0$ depending only on local regularity constants such that
\begin{align}
\mathcal{E}_{\mathrm{evt}}^{(\ell)}(e_t,K)
\leq&
\mathcal{R}_{\mathrm{loc}}^{(\ell)}(K,e_t)
+
\mathcal{R}_{\mathrm{con}}^{(\ell)}(K,e_t)
\nonumber\\
&+
\mathcal{R}_{\mathrm{het}}^{(\ell)}(K,e_t)
+
\mathcal{C}_{\mathrm{stat}}^{(\ell)}(e_t),
\end{align}
where
\begin{equation}
\mathcal{R}_{\mathrm{loc}}^{(\ell)}(K,e_t)
=
q_{\mathrm{opt},\ell}^{2K}
\mathcal{C}_{\mathrm{init}}^{(\ell)}(e_t),
\end{equation}
\begin{align}
\mathcal{R}_{\mathrm{con}}^{(\ell)}(K,e_t)
=&
c_{\mathrm{con},\ell}L_\ell
\bar{\rho}_\ell^{2K}(e_t)
\mathcal{D}_{0}^{(\ell)}(e_t)
\nonumber\\
&+
c_{\mathrm{con},\ell}L_\ell
\eta_c^2 n_\ell^2
\nonumber\\
&\times
\sum_{s=0}^{K-1}
\bar{\rho}_\ell^{2(K-1-s)}(e_t)
\mathcal{I}_{s}^{(\ell)}(e_t),
\end{align}
and
\begin{equation}
\mathcal{R}_{\mathrm{het}}^{(\ell)}(K,e_t)
=
L_\ell
\frac{
c_{\mathrm{het},\ell}\eta_c^2 n_\ell^2
\tilde{\zeta}_\ell^2(e_t)}
{1-q_{\mathrm{opt},\ell}^{2}}.
\end{equation}
Finally, $\mathcal{C}_{\mathrm{stat}}^{(\ell)}(e_t)$ is given by
\begin{equation}
\mathcal{C}_{\mathrm{stat}}^{(\ell)}(e_t)
=
\liminf_{K\to\infty}\;
\min_{\Pi_\ell,\;\mathcal{E}_t^{(\ell)}}
\;
\mathcal{E}_{\mathrm{evt}}^{(\ell)}(e_t, K).
\end{equation}
\end{theorem}
\begin{proof}
    See Appendix \ref{Theorem2}.
\end{proof}
From Theorem~2, we can see that the bound reveals how different factors govern the correction gap based on local residual $\mathcal{R}_{\mathrm{loc}}^{(\ell)}$, consensus residual $\mathcal{R}_{\mathrm{con}}^{(\ell)}$, heterogeneity residual $\mathcal{R}_{\mathrm{het}}^{(\ell)}$, and irreducible residual $\mathcal{C}_{\mathrm{stat}}^{(\ell)}(e_t)$.
The local residual $\mathcal{R}_{\mathrm{loc}}^{(\ell)}$ shrinks as $q_{\mathrm{opt},\ell}^{2K}$ where more local correction steps reduce this term.
The consensus residual $\mathcal{R}_{\mathrm{con}}^{(\ell)}$ shrinks as $\bar{\rho}_\ell^{2K}$ where stronger aggregation (smaller $\bar{\rho}_\ell$) reduces this term.
The heterogeneity residual $\mathcal{R}_{\mathrm{het}}^{(\ell)}$ does not decay with $K$, forming a floor determined by the mismatch among devices' local objectives.
These three residuals are not independent.
In particular, increasing the number of local steps $n_\ell$ reduces $\mathcal{R}_{\mathrm{loc}}^{(\ell)}$ (through $q_{\mathrm{opt},\ell}^{2K}$), but simultaneously increases $\mathcal{R}_{\mathrm{con}}^{(\ell)}$ (through the $\eta_c^2 n_\ell^2$ term in the injected disagreement), because more local steps allow devices to drift further apart before the next aggregation step.
Similarly, stronger aggregation (smaller $\bar{\rho}_\ell$) reduces $\mathcal{R}_{\mathrm{con}}^{(\ell)}$ but consumes communication budget that could otherwise be used for more frequent synchronization.
Finally, the irreducible residual term $\mathcal{C}_{\mathrm{stat}}^{(\ell)}(e_t)$ represents the minimum event correction gap achievable under the optimal policy and topology as the number of correction rounds tends to infinity.
It collects errors that cannot be removed by changing local steps or communication edges as finite-sample noise, model approximation error, and the mismatch between the best attainable adapter under the given data distribution and the event oracle itself.
Theorem~2 applies to any membership event $e_t = (\mathcal{J}_t, \mathcal{U}_t)$.
The two specializations recover standard guarantees for the pure departure and pure join regimes.

To further simplify the modeling, we configure a separate communication topology and correction schedule for each of the $L$ LoRA layers. We partition the $L$ layers into $G$ coarser layer groups based on architectural proximity (e.g., early, middle, and late blocks).
Each group $g$ shares a single edge set $\mathcal{E}_t^{(g)}$, a single aggregation matrix $\bm{W}_t^{(g)}$, and a single correction schedule $\Pi_\ell$.

\subsection{Correction Policy Design}
\textcolor{black}{%
Theorem~2 and Assumption~2 together identify two quantities that govern the correction difficulty of each layer group.
The first is the local Lipschitz constant $L_g$: under Assumption~2, the admissible step size is constrained by $\eta \leq 1/L_g$, so a group with larger $L_g$ advances more slowly per local step.
The second is the initial parameter residual $\|\bm{r}_0^{(g)}\|^2$ from Theorem~1: the PL condition in Assumption~2 relates this to the post-projection gradient energy through $\|\nabla\mathcal{L}(\widetilde{\bm{B}}_0^{(g)})\|_F^2 \geq 2\mu_g \cdot \mathcal{E}_{\mathrm{evt}}^{(g)}(e_t,0)$, so a larger gradient signals a group that starts further from the event oracle.
Consequently, a group with both high curvature and a large initial gap faces steeper correction difficulty across all residual terms $\mathcal{R}_{\mathrm{loc}}$, $\mathcal{R}_{\mathrm{con}}$, and $\mathcal{R}_{\mathrm{het}}$ in Theorem~2.
}

\textcolor{black}{%
Both quantities admit computable proxies from one forward and one backward pass after projection initialization.
To estimate it, we employ the empirical Fisher information matrix $\bm{F}_g^{(e_t)} \in \mathbb{R}^{r_g \times r_g}$ that approximates the Hessian via the Gauss--Newton correspondence~\cite{karimi2016linear}.
Its largest eigenvalue $\lambda_{\max}(\bm{F}_g^{(e_t)})$ preserves the rank order of $L_g$.
The post-projection gradient energy $\frac{1}{|\mathcal{M}_{t+1}|}\sum_{i\in\mathcal{M}_{t+1}} \|\nabla_{\bm{B}^{(g)}}\mathcal{L}_i(\widetilde{\bm{B}}_0)\|_F^2$ reflects the initialization gap and is available from one backward pass.
}
\textcolor{black}{%
Thus, their product provides a rank-order proxy for the correction potential, which is given by
\begin{equation}
\mathcal{S}_g
=
\lambda_{\max}\!\bigl(\bm{F}_g^{(e_t)}\bigr)
\cdot
\frac{1}{|\mathcal{M}_{t+1}|}\sum_{i\in\mathcal{M}_{t+1}}
\bigl\|\nabla_{\bm{B}^{(g)}}\mathcal{L}_i(\widetilde{\bm{B}}_0)\bigr\|_F^2,
\label{eq:Sg_computable}
\end{equation}
where a higher $\mathcal{S}_g$ signals that group $g$ has both steeper curvature and larger initial gap.
The shared allocation proportion follows as
\begin{equation}
p^{(g)} = \frac{\mathcal{S}_g}{\sum_{g'} \mathcal{S}_{g'} + \epsilon},
\label{eq:p_shared}
\end{equation}
with $\epsilon>0$ avoiding division by zero.
}

Given the shared proportion $p^{(g)}$, the correction schedule for layer group $g$ is $\Pi_g = (n_g,\eta_g,\lambda_g^{\mathrm{prox}},\gamma_g)$, where $n_g$ is the number of local correction steps per aggregation operation, $\eta_g$ is the step size, $\lambda_g^{\mathrm{prox}}$ is the proximal damping coefficient, and $\gamma_g$ is the aggregation strength.
Among these, $n_g$ and $\lambda_g^{\mathrm{prox}}$ consume only local computation and incur no communication cost.
In contrast, increasing $\gamma_g$ and adding communication edges consume the per-round budget $C_t$.

\textcolor{black}{%
A group with high $\mathcal{S}_g$ faces three simultaneous difficulties: its loss surface is steep (requiring more local steps $n_g$), its post-projection optimum differs substantially from other devices' (requiring stronger proximal anchoring $\lambda_g^{\mathrm{prox}}$), and correcting these deviations benefits from faster consensus (requiring denser mixing $\gamma_g$).
Scaling all three dimensions by the same rank-order proportion $p^{(g)}$ is the simplest allocation rule that respects this joint dependence while remaining computable from the two one-pass observables in~\eqref{eq:Sg_computable}.
}
For local correction we design $n_g = n_{\min} + (n_{\max} - n_{\min})\,p^{(g)}$.
For heterogeneity control we employ $\lambda_g^{\mathrm{prox}} = \lambda_{\max}\,p^{(g)}$.
Similarly, for consensus aggregation we employ $\gamma_g = \gamma_{\min} + (1 - \gamma_{\min})\,p^{(g)}$.
When $\mathcal{S}_g$ is small, $p^{(g)}$ is near zero and the parameters revert to their minimal-cost defaults.

The remaining budget, after accounting for local steps and proximal damping, is allocated to topology by sampling a connected random graph of target density $\gamma_g$.
The construction requires no per-edge marginal-gain estimation and is compatible with the finite-compute, rank-order regime.
Algorithm~1 summarizes the per-round correction policy.

{\color{black}
\begin{algorithm}[t]
\caption{Priority-Aware Per-Round Correction Policy}
\label{alg:correction_policy}
\begin{algorithmic}[1]
\STATE \textbf{Input}: Event $e_t$, active devices $\mathcal{M}_{t+1}$, communication budget $C_t$.
\STATE \textbf{Output}: Updated $\bm{B}_{i,t}^{(g)}$ for all $i\in\mathcal{M}_{t+1}$ and all $g$.
\STATE Compute post-event initialization $\widetilde{\bm{B}}_0^{(g)}$.

\FOR{each layer group $g$}
    \STATE Compute $\mathcal{S}_g$ from observables via Eq.~\eqref{eq:Sg_computable} (one forward + one backward pass).
\ENDFOR
\STATE Compute the shared proportion $p^{(g)} = \mathcal{S}_g/(\sum_{g'}\mathcal{S}_{g'}+\epsilon)$ for each $g$ via Eq.~\eqref{eq:p_shared}.
\STATE Set $n_g$, $\gamma_g$, and $\lambda_g^{\mathrm{prox}}$ via the allocation rules in Sec.~IV-C.
\FOR{each layer group $g$}
    \STATE Sample a connected random graph of density $\gamma_g$ and form the Metropolis-weighted mixing matrix $\bm{W}_t^{(g)}$.
\ENDFOR
\STATE Run $K$ rounds of DGD correction under $(\Pi_g,\bm{W}_t^{(g)})$.
\end{algorithmic}
\end{algorithm}
}
\section{Numerical Validation}

\subsection{Experimental Setup}\label{subsec:exp_setup}

We evaluate the proposed framework on a decentralized LoRA fine-tuning testbed that includes Qwen-7B \cite{yang2025qwen3} for QA tasks.
We also evaluate its performance on the ResNet-18 model with the CIFAR-100 dataset, where each device owns a subset of the training data.
The base model is frozen, and each device trains and exchanges only low-rank adapter matrices as modeled.
To induce non-trivial per-block heterogeneity, layers are partitioned into three groups $g\in\{\mathrm{early},\mathrm{mid},\mathrm{late}\}$ with LoRA ranks $r_g\in\{4,8,16\}$.
The system uses $M=6$ devices communicating over a sparse mesh, and the correction horizon is $K=60$ rounds after the membership event.

Table~\ref{Tab:qwen_policy} reports the per-block priority diagnosis on the constructed stress test.
Late-block $S_g$ is $227\%$ higher than early-block $S_g$.
The policy $\Pi_g$ that the proposed method derives from $S_g$ doubles local steps and adds a $10^{-3}$ proximal anchor specifically on the high-$S_g$ block.
The high-$S_g$ block dominates the correction residual, so allocating extra local steps and proximal damping there reduces the residual fastest per unit of compute, while leaving low-$S_g$ blocks at uniform settings keeps cost low.

\begin{table}[t]
\centering
\caption{Per-block priority signals and policy $\Pi_g$ assigned by Ours (full).}
\label{Tab:qwen_policy}
\resizebox{0.5\textwidth}{!}{
\begin{tabular}{lcccccccc}
\hline
Block & $r_g$ & $S_g$ & $\hat S_g$ & Class & $n_g$ & $\eta_g$ & $\lambda_g^{\mathrm{prox}}$ & $\gamma_g$\\
\hline
Early & 4 & 0.6346 & 0.497 & low & 1 & $5.0\!\times\!10^{-4}$ & $0$ & $0.40$\\
Mid & 8 & 1.1394 & 0.893 & mid & 1 & $5.0\!\times\!10^{-4}$ & $0$ & $0.40$\\
Late & 16 & 2.0552 & 1.610 & high & 2 & $5.0\!\times\!10^{-4}$ & $1.0\!\times\!10^{-3}$ & $0.40$\\
\hline
\end{tabular}
}
\end{table}


\subsection{Baseline Comparison}\label{subsec:baselines}

We compare the proposed method against five algorithms:
\begin{enumerate}
    \item FedEraser-D2D~\cite{federaser}: per-device approximation of the leaver's gradient using each device's own data on the leaver's basis $\bm{A}_u$.
    \item EWC-leave / EWC-Join~\cite{kirkpatrick2017ewc}: Fisher-weighted proximal anchor to the pre-event consensus $\bm{B}^*_{\mathrm{pre}}$ with post-hoc Fisher diagonal.
    \item Influence-Fn-D2D~\cite{koh2017influence}: block-diagonal empirical Gauss-Newton inverse (natural-gradient rescaling).
    \item KD-Unlearn~\cite{kelsch2025spare}: KL toward base-model logits, the only teacher available in D2D-LoRA.
    \item Naive Remove / Fine-tune: uniform DGD on remaining devices without projection (leave) / standard DGD with the new device (join).
\end{enumerate}

Fig.~\ref{fig_leave_baselines} shows the event gap $E_{FU}$ over correction rounds for the proposed method and five baselines under the leave-only membership event.
From this figure, we can see that the proposed method achieves a $6.3\%$ improvement compared to Naive Remove and a $22\%$ improvement compared to KD-Unlearn on the final gap.
This gain comes from the priority-aware policy $\Pi_g$ concentrating local-step budget and proximal damping on the high-$S_g$ block, attacking the dominant residual.
We can also see that the KD-Unlearn plateaus at the no-correction floor and the proposed method achieves a $4\%$ improvement compared to the four-baseline cluster on the final gap.
This failure mode arises because the only D2D-available teacher is the frozen base model with no task knowledge, so KL toward it pulls the adapter back toward $\bm{B}{=}0$ rather than removing the leaver's contribution.

\begin{figure*}[t]\centering
\includegraphics[width=0.75\textwidth]{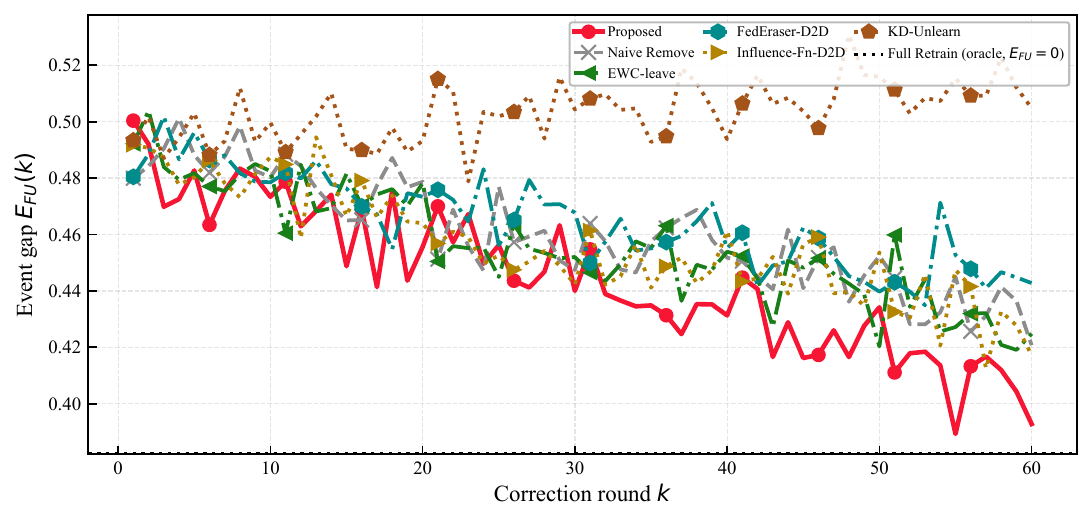}
\caption{Event gap $E_{FU}$ vs round $k$ for the proposed method and five baselines (leave-only).}\label{fig_leave_baselines}
\end{figure*}

In Fig. \ref{fig_RI}, we report the per-round relative improvement $\mathrm{RI}(k) = 100 \times (L_{\mathrm{uniform}}(k) - L_{\mathrm{method}}(k)) / L_{\mathrm{uniform}}(k)$, where $L_{\mathrm{uniform}}(k)$ and $L_{\mathrm{method}}(k)$ are the consensus retain CE losses of the Uniform topology and resource allocation anchor and the evaluated method at step $k$, respectively.
In particular, a positive RI indicates the method outperforms Uniform.
Per-block density follows the same $\mathcal{S}_g$-driven allocator, with $n_{\max}{=}n_{\min}{=}1$ and $\lambda_{\max}{=}0$ so the shared proportion $p^{(g)}$ enters only the mixing density $\gamma_g$.
From this figure, we can first see that the proposed method achieves a $0.53\%$ improvement in final-round RI compared to Uniform, while FedEraser-D2D achieves a $0.69\%$ improvement at the cost of $109$\,KB per-device gradient bookkeeping, so the proposed method matches FedEraser-D2D's correction quality with zero storage overhead.
This near-parity comes from the priority-aware $\gamma_g$ concentrating mixing on the high-$\mathcal{S}_g$ block, which absorbs the same per-event residual that FedEraser-D2D's stored per-device gradients target, but without storing the cumulative contribution of every device.
Second, we can also see that the three Fisher- and influence-based correctors (EWC-leave $-3.50\%$, Influence-Fn-D2D $-3.46\%$, and KD-Unlearn $-0.17\%$) all underperform Uniform, while Naive Remove drops to $-3.78\%$ and never recovers.
This collapse confirms that anchoring the post-event correction to a pre-event reference — whether through a Fisher-weighted proximal, a Gauss-Newton preconditioner, or a frozen-teacher KL — is a sharp floor on the dynamic timeline once the basis is projected, because the reference no longer exists in the parameter space the corrected $\bm{B}$ now uses, consistent with Fig.~\ref{fig_leave_baselines}.

\begin{figure}[t]\centering
\includegraphics[width=\columnwidth]{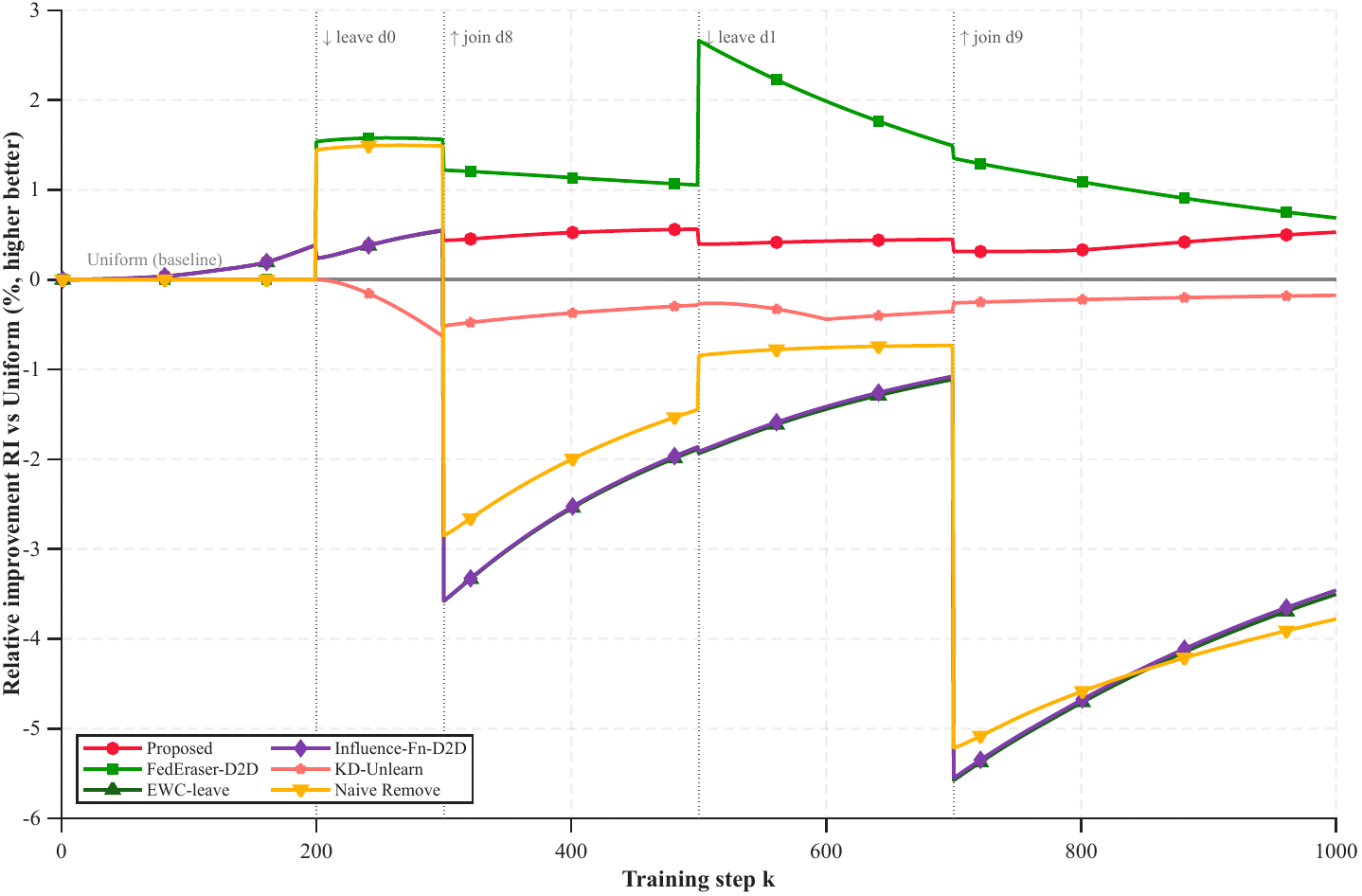}
\caption{Per-round relative improvement RI (\%) vs Uniform over $K{=}1000$ rounds on the CIFAR-100 + ResNet-18 testbed; vertical dashed lines mark the four membership events at $k{\in}\{200,300,500,700\}$.}\label{fig_cifar_baselines_ri}
\label{fig_RI}
\end{figure}

Fig.~\ref{fig_forget_retain} shows the per-round retain loss over correction rounds.
From this figure, we can first see that the proposed method achieves a $1\%$ improvement compared to the four scheduling-agnostic baselines (Naive Remove, EWC-leave, FedEraser-D2D, Influence-Fn-D2D) on the final retain loss.
This smaller retain damage comes from $\Pi_g$ concentrating correction effort on the high-$S_g$ block and sparing the low-$S_g$ blocks from over-update.
Second, the proposed method achieves a $4\%$ improvement over KD-Unlearn, which remains the highest-performing curve throughout.
The worst-case damage of KD-Unlearn stems from its KL target being the task-agnostic base model, which pulls the adapter toward $\bm{B}{=}0$ and erases the retained knowledge accumulated during pretraining.

\begin{figure}[t]\centering
\includegraphics[width=\columnwidth]{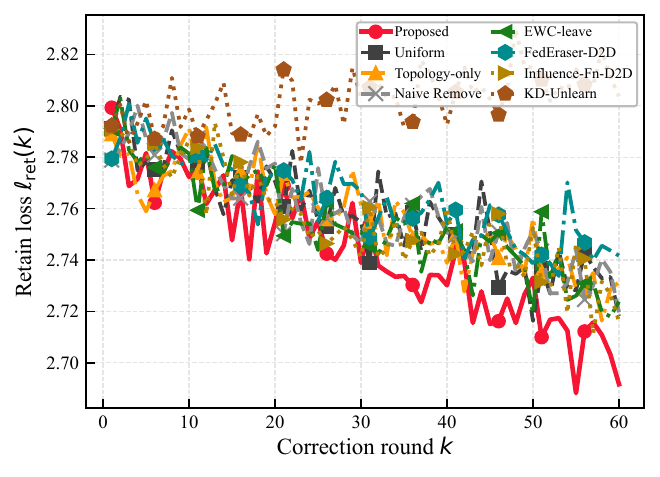}
\caption{Retain loss $\ell_{\mathrm{ret}}$ vs round $k$ (leave-only event).}\label{fig_forget_retain}
\end{figure}


\subsection{Ablation Study}\label{subsec:ablation}

To conduct the ablation study, we employ six variants of the proposed method, denoted by the prefix “Ours,” each removing one or more components of the priority-aware policy $\Pi_g{=}(n_g,\eta_g,\lambda_g^{\mathrm{prox}},\gamma_g)$ (with sync period $H_g$):
\begin{itemize}
    \item Ours (full): full priority-aware policy ($n_g$, $\lambda_g^{\mathrm{prox}}$, $H_g$, $\gamma_g$ all derived from $S_g$).
    \item Ours (ls+prox): priority-aware $n_g$ and $\lambda_g^{\mathrm{prox}}$; $H_g{=}1$ and $\gamma_g{=}0.40$ uniform.
    \item Ours (local-steps): priority-aware $n_g$ only.
    \item Ours (retain-prox): priority-aware $\lambda_g^{\mathrm{prox}}$ only.
    \item Ours (uniform): no priority-aware component (baseline ablation).
    \item Ours (topology): $S_g$ routed only to $\gamma_g$ (Section~IV-D shows this is the wrong knob).
\end{itemize}

Fig.~\ref{fig_qwen_convergence} shows $E_{FU}$ over correction rounds for the six ablation arms.
First, Ours (full) achieves a $6.3\%$ improvement compared to Ours (uniform) on the final gap.
This gain comes from stacking three priority-aware components (local-steps, retain-prox, priority synchronization) on the high-$S_g$ block.
Second, Ours (full) achieves a $9.2\%$ improvement compared to Ours (topology), the only arm that lags Ours (uniform).
This regression of Ours (topology) comes from routing $S_g$ to edge density alone, which adds communication without reducing the local-optimization residual.

\begin{figure}[t]\centering
\includegraphics[width=\columnwidth]{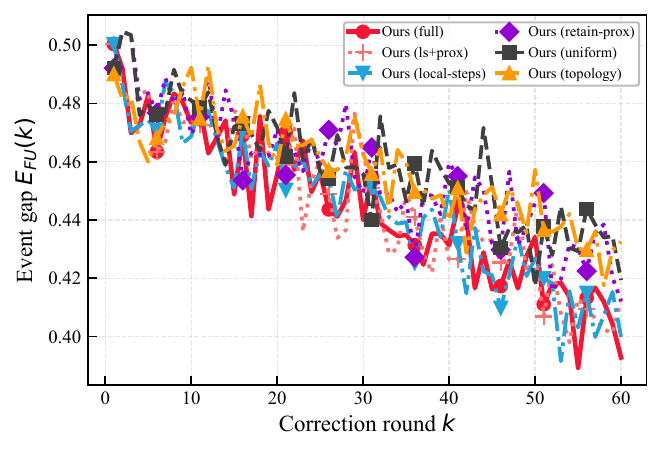}
\caption{Event gap $E_{FU}$ vs round $k$ for six ablation arms (leave-only).}\label{fig_qwen_convergence}
\end{figure}

Fig.~\ref{fig_qwen_cost} shows each ablation arm's trajectory in the ($E_{FU}$, normalized cost) plane.
First, Ours (full) ends Pareto-dominant, with the lowest gap and the lowest cost.
This comes from combining gap-reducing scheduling (local-steps + retain-prox) with cost-reducing priority synchronization, which no single ablation arm achieves alone.
Second, Ours (full) achieves a $9.2\%$ improvement compared to Ours (topology) on the final gap and a $27\%$ cost reduction compared to Ours (topology) at $K{=}60$.
This double-loss of Ours (topology) comes from the same root cause as Fig.~\ref{fig_qwen_convergence}: extra edges cost communication, but the priority signal alone cannot reduce the local-optimization residual.

\begin{figure}[t]\centering
\includegraphics[width=\columnwidth]{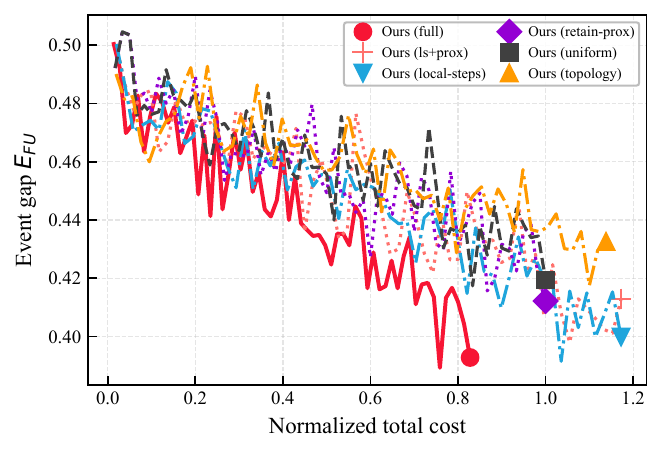}
\caption{Event gap $E_{FU}$ vs normalized cost for six ablation arms (leave-only).}\label{fig_qwen_cost}
\end{figure}

Fig.~\ref{fig_qwen_ri} shows the per-round RI of each ablation arm relative to Ours (uniform).
First, priority-aware arms (Ours (local-steps), Ours (retain-prox), Ours (ls+prox), Ours (full)) cross zero by round 15 and stay above thereafter.
This crossover occurs because the high-$S_g$ block requires 15 additional local steps before the proximal anchor takes hold; once it does, the gap separation is stable.
Second, Ours (uniform) achieves a $3\%$ improvement compared to Ours (topology) throughout the run, with Ours (topology) staying persistently below zero.
This persistent negative RI of Ours (topology) stems from extra edges that impede communication without any corresponding gap reduction, so the per-round RI never recovers.

\begin{figure}[t]\centering
\includegraphics[width=\columnwidth]{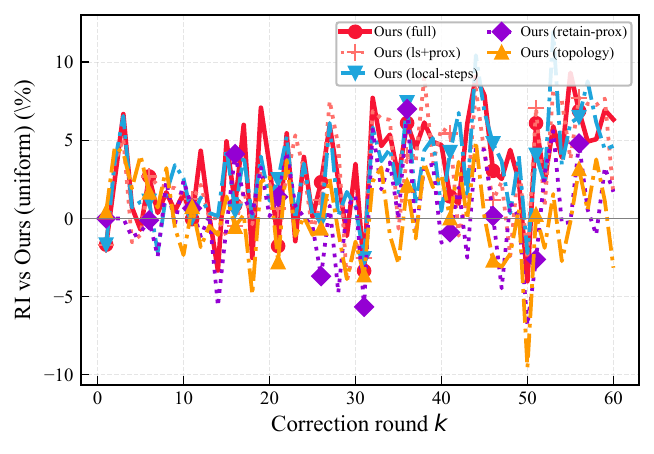}
\caption{RI vs round $k$ for five ablation arms (baseline Ours (uniform)).}\label{fig_qwen_ri}
\end{figure}

\begin{table*}[t]\centering
\caption{Final/best gap $E_{FU}$, RI, normalized cost, and stability for six ablation arms (leave-only, $K{=}60$).}\label{Tab:qwen_main}
\begin{tabular}{lcccccc}\hline
Method & Final & Best & RI vs Ours (uniform)\,(\%) & NormComm & NormTotal & Stab\\\hline
Ours (full)         & $0.39$ & $0.39$ & $+6.3$ & $0.58$ & $0.83$ & oscillatory\\
Ours (ls+prox)      & $0.41$ & $0.40$ & $+1.5$ & $1.00$ & $1.17$ & oscillatory\\
Ours (local-steps)  & $0.40$ & $0.39$ & $+4.7$ & $1.00$ & $1.17$ & oscillatory\\
Ours (retain-prox)  & $0.41$ & $0.41$ & $+1.7$ & $1.00$ & $1.00$ & oscillatory\\
Ours (uniform)      & $0.42$ & $0.42$ & $\phantom{+}0.0$ & $1.00$ & $1.00$ & monotone\\
Ours (topology)     & $0.43$ & $0.42$ & $-3.1$ & $1.17$ & $1.14$ & oscillatory\\
No-correction (ref)        & $0.50$ & $0.50$ & $-20.2$ & $0.00$ & $0.00$ & flat\\
\hline\end{tabular}\end{table*}


\section{Conclusion}
This paper studies a decentralized collaborative fine-tuning framework with device joins and leaves dynamically.
To enhance learning-unlearning performance, we proposed a frozen random orthogonal basis mechanism that provides a no-history contribution index, enabling projection deletion after leave events. 
By separating the knowledge subspace before membership event changes, our proposed method can achieve fast convergence through subspace initialization.
The theoretical analysis shows that finite-round correction is governed by local optimization, consensus aggregation, and heterogeneity drift residuals. 
Then, a priority-aware learning-unlearning correction resource allocation method is proposed that enables priority optimization according to different layer groups.
The resulting framework provides a principled way to support fast learning-unlearning correction in dynamic decentralized LoRA systems under per-round communication constraints.

\appendix

\subsection{Proof of Lemma 1}
\label{Lemm1}
Fix layer group $g$. Since $\bm{A}_i^{(g)}$ and $\bm{A}_j^{(g)}$ are obtained by orthogonalizing independent Gaussian matrices, their column spaces are independently and uniformly distributed on the Stiefel manifold. Let $\bm{a}_{i,p}^{(g)}$ be the $p$-th column of $\bm{A}_i^{(g)}$. Then
\begin{equation}
\left\|
\left(\bm{A}_i^{(g)}\right)^{\top}
\bm{A}_j^{(g)}
\right\|_F^2
=
\sum_{p=1}^{r_g}
\left\|
\left(\bm{A}_j^{(g)}\right)^{\top}
\bm{a}_{i,p}^{(g)}
\right\|_2^2 .
\end{equation}
Conditioned on $\bm{a}_{i,p}^{(g)}$, the expected squared projection of a unit vector onto an independent $r_g$-dimensional random subspace of $\mathbb{R}^{d_g}$ is $r_g/d_g$. Summing over $p=1,\ldots,r_g$ gives
\begin{equation}
\mathbb{E}
\left[
\left\|
\left(\bm{A}_i^{(g)}\right)^{\top}
\bm{A}_j^{(g)}
\right\|_F^2
\right]
=
\frac{r_g^2}{d_g}.
\end{equation}
\subsection{Proof of Lemma 2}
\label{Lemm2}
By definition,
\begin{equation}
\bm{P}_j^{(g)}
=
\bm{A}_j^{(g)}
\left(\bm{A}_j^{(g)}\right)^{\top}.
\end{equation}
Therefore,
\begin{align}
\bm{B}_{i,t}^{(g)}
\left(\bm{A}_i^{(g)}\right)^{\top}
\bm{P}_j^{(g)}
=&
\bm{B}_{i,t}^{(g)}
\left(\bm{A}_i^{(g)}\right)^{\top}
\bm{A}_j^{(g)}
\left(\bm{A}_j^{(g)}\right)^{\top}.
\end{align}
Since $\bm{A}_j^{(g)}$ has orthonormal columns, $\|(\bm{A}_j^{(g)})^{\top}\|_2=1$. The submultiplicativity of the Frobenius norm gives
\begin{equation}
\left\|
\bm{B}_{i,t}^{(g)}
\left(\bm{A}_i^{(g)}\right)^{\top}
\bm{P}_j^{(g)}
\right\|_F
\leq
\|\bm{B}_{i,t}^{(g)}\|_F
\left\|
\left(\bm{A}_i^{(g)}\right)^{\top}
\bm{A}_j^{(g)}
\right\|_2 .
\end{equation}
Squaring both sides proves the lemma.
\subsection{Proof of Lemma 3}
\label{Lemm3}
Since $\bm{W}$ is doubly stochastic and $\bm{J}$ is the projection onto $\bm{1}$, we have $\bm{W}\bm{J} = \bm{J}\bm{W} = \bm{J}$.
For any $\bm{x}$,
\begin{align}
& \|\bm{W}\bm{x} - \bm{J}\bm{x}\|_F \\
&= \|\bm{W}\bm{x} - \bm{J}\bm{W}\bm{x}\|_F \nonumber\\
&= \|(\bm{W} - \bm{J})\bm{W}\bm{x}\|_F \nonumber\\
&= \|(\bm{W} - \bm{J})(\bm{x} - \bm{J}\bm{x})\|_F
\leq \|\bm{W} - \bm{J}\|_2 \cdot \|\bm{x} - \bm{J}\bm{x}\|_F,
\end{align}
where the last line uses $\bm{W}\bm{x} - \bm{J}\bm{x} = (\bm{W} - \bm{J})(\bm{x} - \bm{J}\bm{x})$ (since $(\bm{W} - \bm{J})\bm{J} = \bm{0}$) and the submultiplicativity of the spectral norm with the Frobenius norm.
The second claim ($\|\bm{W} - \bm{J}\|_2 < 1$ for a connected graph under the Metropolis rule) is standard, and follows from the Perron-Frobenius theorem for primitive stochastic matrices.
\subsection{Proof of Theorem 2}
\label{Theorem2}

Define the per-round centralized optimization error and the device disagreement as
\begin{align}
\Delta_k &= \bigl\|\bar{\bm{B}}_k^{(\ell)} - \bm{B}^{(\ell),\star}(e_t)\bigr\|_2^2, \label{eq:prf_delta}
\end{align}
and 
\begin{align}
D_k &= \frac{1}{|\mathcal{M}{t+1}|}\sum{i\in\mathcal{M}{t+1}} \bigl\|\bm{B}{i,k}^{(\ell)} - \bar{\bm{B}}_k^{(\ell)}\bigr\|_2^2. \label{eq:prf_D}
\end{align}
By $L$-smoothness (Assumption~2), the event loss gap is bounded by the parameter distance:
\begin{equation}\label{eq:prf_smooth}
\mathcal{E}_{\mathrm{evt}}^{(\ell)}(e_t,K) \leq \frac{L}{2}\,\Delta_K.
\end{equation}

Between two successive aggregation rounds, each device performs $n$ local gradient steps. Under the $\mu$-PL condition and $L$-smoothness (Assumption~2), $n$ consecutive gradient steps contract the centralized optimality gap by $(1-\eta_c\mu)^n = q$~\cite{karimi2016linear}. Coupling local contraction, objective mismatch, and disagreement via standard DGD analysis~\cite{yuan2016convergence} yields
\begin{equation}\label{eq:prf_delta_recur}
\Delta_{k+1} \leq q^{2}\,\Delta_k + c_{1}\,\eta_c^{2} n^{2}\,\tilde{\zeta}^{2} + c_{2}\,D_k.
\end{equation}

The aggregation step with mixing matrix $\bm{W}t^{(\ell)}$ contracts disagreement at rate $\bar{\rho} = \max{0\leq t<K}\|\bm{W}_t^{(\ell)} - \bm{J}\|_2 < 1$ (Lemma~3). The per-round disagreement evolves as
\begin{equation}\label{eq:prf_D_recur}
D_{k+1} \leq \bar{\rho}^{2}\,D_k + c_{3}\,\eta_c^{2} n^{2}\,\mathcal{I}_k.
\end{equation}

Iterating~\eqref{eq:prf_D_recur} from round $0$ to $k-1$ gives
\begin{equation}\label{eq:prf_D_unrolled}
D_k \leq \bar{\rho}^{2k} D_0 + c_{3}\,\eta_c^{2} n^{2} \sum_{s=0}^{k-1} \bar{\rho}^{2(k-1-s)} \mathcal{I}_s.
\end{equation}

Then, substituting~\eqref{eq:prf_D_unrolled} into~\eqref{eq:prf_delta_recur} and iterating over $K$ rounds,
\begin{align}
\Delta_K &\leq q^{2K}\Delta_0 \nonumber\\
&\quad + c_{1}\eta_c^{2} n^{2}\tilde{\zeta}^{2} \sum_{k=0}^{K-1} q^{2(K-1-k)} \nonumber\\
&\quad + c_{2} \sum_{k=0}^{K-1} q^{2(K-1-k)} D_k. \label{eq:prf_delta_unrolled}
\end{align}

Since $q^{2}<1$ and $\bar{\rho}^{2}<1$, the cross term $q^{2(K-1-k)}\bar{\rho}^{2k}$ in the third line of~\eqref{eq:prf_delta_unrolled} is bounded by $\max\{q^{2},\bar{\rho}^{2}\}^{K-1}$ up to a multiplicative constant that depends only on the ratio $q^{2}/\bar{\rho}^{2}$. Absorbing this factor into $c_{\mathrm{con},\ell}$, the third line simplifies to
\begin{align}
c_{\mathrm{con},\ell}\, \bar{\rho}^{2K} D_0
\;+\;
c_{\mathrm{con},\ell}\, \eta_c^{2} n^{2} \sum_{s=0}^{K-1} \bar{\rho}^{2(K-1-s)} \mathcal{I}_s. \label{eq:prf_consensus}
\end{align}

Multiply~\eqref{eq:prf_delta_unrolled} by $L/2$ via~\eqref{eq:prf_smooth} and group terms.

\smallskip\noindent\textit{Local residual.} From the first term of~\eqref{eq:prf_delta_unrolled}, $(L/2)\,q^{2K}\Delta_0$. Theorem~1 gives $\mathcal{E}_{\mathrm{evt}}^{(\ell)}(e_t,0) \leq (L/2)\|\bm{r}_0^{(\ell)}\|_2^{2}$. The initial average distance satisfies $\Delta_0 = \|\bar{\bm{B}}_0^{(\ell)} - \bm{B}^{(\ell),\star}\|_2^{2} \leq \|\bm{r}_0^{(\ell)}\|2^{2}$ (the average does not exceed the maximum per-device residual). Defining $\mathcal{C}{\mathrm{init}}^{(\ell)}(e_t) \triangleq \frac{L}{2}\|\bm{r}_0^{(\ell)}\|_2^{2}$, we obtain
\begin{equation}
\frac{L}{2}\,q^{2K}\Delta_0 \;\leq\; q^{2K}\,\mathcal{C}_{\mathrm{init}}^{(\ell)}(e_t) \;\equiv\; \mathcal{R}_{\mathrm{loc}}^{(\ell)}(K,e_t).
\end{equation}

From the second term of~\eqref{eq:prf_delta_unrolled}. The geometric series evaluates to $\sum_{k=0}^{K-1} q^{2(K-1-k)} = (1-q^{2K})/(1-q^{2}) \leq 1/(1-q^{2})$, which does not vanish with $K$ and forms a persistent floor. Setting $c_{\mathrm{het},\ell} \triangleq c_{1}/2$,
\begin{equation}
\frac{L}{2}\cdot\frac{c_{1}\eta_c^{2} n^{2}\tilde{\zeta}^{2}}{1-q^{2}}
\;=\;
L\,\frac{c_{\mathrm{het},\ell}\,\eta_c^{2} n^{2}\tilde{\zeta}^{2}}{1-q^{2}}
\;\equiv\;
\mathcal{R}_{\mathrm{het}}^{(\ell)}(K,e_t).
\end{equation}

From the third term, after the simplification in~\eqref{eq:prf_consensus} and multiplying by $L/2$, with $\mathcal{D}_0$ substituting $D_0$ (they are identical by definition~\eqref{eq:prf_D}),
\begin{align}
\mathcal{R}_{\mathrm{con}}^{(\ell)}(K,e_t)
&\triangleq\;
c_{\mathrm{con},\ell}L\,
\bar{\rho}^{2K} \mathcal{D}_0^{(\ell)}(e_t) \nonumber\\
&\quad +\; c_{\mathrm{con},\ell}L\,
\eta_c^{2} n^{2}\sum_{s=0}^{K-1}
\bar{\rho}^{2(K-1-s)} \mathcal{I}_s^{(\ell)}(e_t).
\end{align}

Finally, we define irreducible residual as
\begin{equation}
\mathcal{C}_{\mathrm{stat}}^{(\ell)}(e_t)
\;\triangleq\;
\liminf_{K\to\infty}\;
\min_{\Pi_\ell,\;\mathcal{E}_t^{(\ell)}}
\mathcal{E}_{\mathrm{evt}}^{(\ell)}(e_t,K),
\end{equation}
which collects errors that no finite-round policy can remove: finite-sample noise, model approximation error, and the inherent mismatch between the best attainable adapter and the event oracle.

\medskip\noindent Summing the four residuals, we obtain
\begin{equation}
\mathcal{E}_{\mathrm{evt}}^{(\ell)}(e_t,K)
\!\!\;\leq\;\!\!
\mathcal{R}_{\mathrm{loc}}^{(\ell)}(K,e_t)
+ \mathcal{R}_{\mathrm{con}}^{(\ell)}(K,e_t)
+ \mathcal{R}_{\mathrm{het}}^{(\ell)}(K,e_t)
+ \mathcal{C}_{\mathrm{stat}}^{(\ell)}(e_t).
\end{equation}
This ends the proof.

\bibliographystyle{IEEEtran}
\renewcommand{\baselinestretch}{0.9}
\bibliography{zotero_dynamic_refs,tnse_extra_refs}

\end{document}